\documentclass{article}

\PassOptionsToPackage{sort&compress}{natbib}

\usepackage[final]{neurips_2020}

\usepackage[utf8]{inputenc} %
\usepackage[T1]{fontenc}    %

\usepackage{xcolor}
\usepackage[unicode=true]{hyperref}       %
\definecolor{mydarkblue}{rgb}{0,0.08,0.45}
\hypersetup{ %
    pdftitle={},
    pdfauthor={},
    pdfsubject={},
    pdfkeywords={},
    pdfborder=0 0 0,
    pdfpagemode=UseNone,
    colorlinks=true,
    linkcolor=mydarkblue,
    citecolor=mydarkblue,
    filecolor=mydarkblue,
    urlcolor=mydarkblue,
    pdfview=FitH
}
    
\usepackage{url}            %
\usepackage{booktabs}       %
\usepackage{amsfonts}       %
\usepackage{nicefrac}       %
\usepackage{microtype}      %
\usepackage{wrapfig}
\usepackage{makecell}

\usepackage{amssymb,amsmath,amsthm}
\usepackage{bbm}
\usepackage{xifthen}
\usepackage{xcolor}
\usepackage{caption}
\usepackage{subcaption}
\usepackage{macros} %
\usepackage{float}
\newsavebox{\imagebox}

\newtheorem{theorem}{Theorem}

\newtheorem{definition}{Definition}

\newtheorem{conjecture}{Conjecture}

\def\shownotes{1}  \ifnum\shownotes=1
\newcommand{\authnote}[2]{[ #2 --#1 ]}
\usepackage{todonotes}
\else
\newcommand{\authnote}[2]{}
\usepackage[disable]{todonotes}
\fi

\title{
Distributional Generalization:\\
A New Kind of Generalization
}

\author{%
  Preetum Nakkiran\thanks{Co-first authors. Author contributions in Appendix~\ref{sec:contrib}.}\\
  Harvard University\\
  \texttt{preetum@cs.harvard.edu}
   \And
   Yamini Bansal\footnotemark[1]\\
  Harvard University\\
  \texttt{ybansal@g.harvard.edu}
}

\begin{document}
\maketitle

\begin{abstract}
We introduce a new notion of generalization--- Distributional Generalization---
which roughly states that outputs of a classifier at train and test time
are close \emph{as distributions}, as 
opposed to close in just their average error.
For example, if we mislabel 30\% of dogs as cats in the train set of CIFAR-10,
then a ResNet trained to interpolation will
in fact mislabel roughly 30\% of dogs as cats on the \emph{test set} as well, while leaving other classes unaffected. 
This behavior is not captured by classical generalization,
which would only consider the average error and not 
the distribution of errors over the input domain.
Our formal conjectures, which are much more general than this example, characterize the form
of distributional generalization that can be expected
in terms of problem parameters:
model architecture, training procedure, number of samples, and data distribution.
We give empirical evidence for these conjectures across
a variety of domains in machine learning,
including neural networks, kernel machines, and decision trees.
Our results thus advance our empirical understanding
of interpolating classifiers.

\end{abstract}

\newpage

\section{Introduction}
We begin with an experiment motivating the need for
a notion of generalization beyond test error.

\begin{experiment}
\label{exp:intro1}
Consider a binary classification version of CIFAR-10,
where CIFAR-10 images $x$ have binary labels \texttt{Animal/Object}.
Take 50K samples from this distribution as a train set, but apply the following
label noise: flip the label of cats to \texttt{Object} with probability 30\%.
Now train a WideResNet $f$ to 0 train error on this train set. How does the trained classifier behave on test samples?
Some potential options are:
\end{experiment}

\begin{enumerate}
\item The test error is uniformly small across all CIFAR-10 classes,
since there is only 3\% overall label noise in the train set.%

\item The test error is moderate, and ``spread'' across all animal classes.
After all, the classifier is not explicitly told what a cat or a dog is, just that they are all animals.

\item The test error is localized:
the classifier misclassifies roughly 30\% of test cats as ``objects'',
but all other types of animals are largely unaffected.
\end{enumerate}
In fact, reality is closest to option (3), for essentially
any good architecture trained to interpolation.
Figure \ref{fig:intro} shows the results of this experiment with a WideResNet \citep{zagoruyko2016wide}.
The left panel shows the joint density
of $(x, y)$ of inputs $x$ and labels $y \in \{\texttt{Object/Animal}\}$
on the train set.
Since the classifier $f$ is interpolating,
this joint distribution is identical to the
classifier's outputs $(x, f(x))$ on the
train set.
The right panel shows the joint density of
$(x, f(x))$
of the classifier's predictions on \emph{test inputs} $x$.

\begin{figure}[h]
    \centering
      \includegraphics[width=\textwidth]{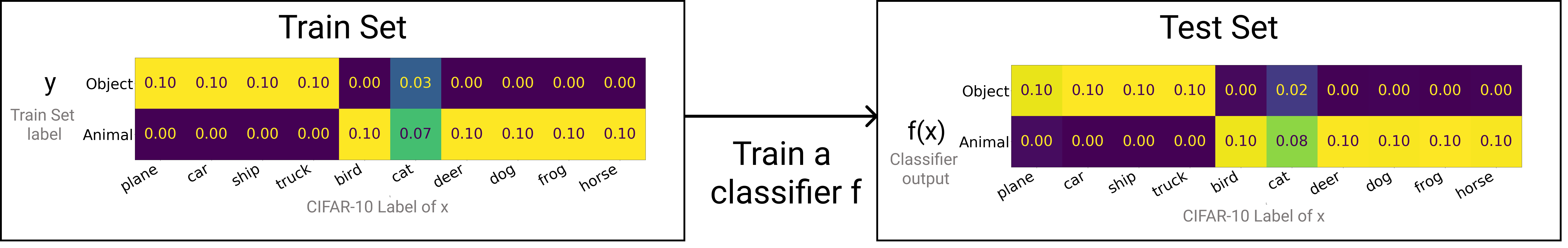}
    \caption{
    The setup and result of Experiment~\ref{exp:intro1}.
    The CIFAR-10 train set is labeled as either Animals or Objects, with label noise affecting only cats.
    A WideResNet-28-10 is then trained to 0 train error on this train set, and evaluated on the test set.
    The joint distribution of $(x, f(x))$ on the train set is close to $(x, f(x))$ on the test set.
    Full experimental details in Appendix \ref{app:intro-exp-1}.
    }
    \label{fig:intro}
\end{figure}

This experiment is interesting for several reasons.
First, the error is \emph{localized} to cats in test set as it was in the train set, even though no explicit cat labels were provided.
Second, the \emph{amount} of error on the cat class is close to the noise applied on the train set.
Thus, the behavior of the classifier on the train set \emph{generalizes} to the test set in a certain sense.
Third, the trained classifier does \emph{not} behave
close to the Bayes optimal classifier for the distribution\footnote{Nor do we expect it to approach the Bayes optimal one
when both data and model size tend to infinity together,
as long as the models remain interpolating.}.
However, the classifier does behave close to an optimal \emph{sampler}
from the conditional distribution, i.e.
$f(x) \sim p(y | x)$.

This behavior would not be captured by solely
considering the average test error --- it requires reasoning about the entire distribution of classifier outputs.
In our work, we show that this experiment is just one instance of a different type of generalization, which we call ``Distributional Generalization''.
We first describe the mathematical form of this generalization.
Then, through extensive experiments, we will show that this type of generalization occurs widely in existing machine learning methods on real datasets, including neural networks, kernel machines and decision trees.

\subsection{Distributional Generalization}

Supervised learning aims to learn a model that correctly classifies inputs $x \in \cX$ from a
given distribution $\cD$ into classes
$y \in \cY$.
We want a model with small \emph{test error} on this distribution.
In practice, we find such a classifier by minimizing
the \emph{train error} of a model on the train set.
This procedure is justified when we expect a small generalization gap:
the gap between the error on the train and test set.
That is, the trained model $f$ should have:
$\textrm{Error}_{\textrm{TrainSet}}(f)
\approx
\textrm{Error}_{\textrm{TestSet}}(f)$.
We now re-write this classical notion of generalization in a form better suited for our extension.\\
{\bf Classical Generalization:}
{\it
Let $f$ be a trained classifier. Then $f$ generalizes if:
\begin{align}
\label{eqn:classic}
\E_{\substack{
x \sim \TrainSet\\
\hat y \gets f(x)
}}
[
\1\{\hat{y} \neq y(x)\}
]
\approx
\E_{\substack{
x \sim \TestSet\\
\hat y \gets f(x)
}}
[
\1\{\hat{y} \neq y(x)\}
]
\end{align}
}
Above, $y(x)$ is the true class of $x$ and $\hat{y}$ is the predicted class.
The LHS of Equation~\ref{eqn:classic} is the train error of $f$,
and the RHS is the test error.
Crucially, both sides of Equation~\ref{eqn:classic}
are expectations of the same function 
($\Terr(x, \hat y) :=  \1\{\hat{y} \neq y(x)\}$)
under different distributions.
The LHS of Equation~\ref{eqn:classic} is the expectation of $\Terr$
under the ``Train Distribution'' $\Dtr$, which
is the distribution over $(x, \hat{y})$ given by sampling a train point $x$
along with its classifier-label $f(x)$.
Similarly, the RHS is under the ``Test Distribution'' $\Dte$, which is this same construction over the test set.
These two distributions are the central objects in our study,
and are defined formally in Section~\ref{sec:general}.
We can now introduce Distributional Generalization, which is a property of trained classifiers.
It is parameterized by a set of bounded functions (``tests''):
$\mathcal{T} \subseteq \{T: \cX \x \cY \to [0, 1]\}$.

{\bf Distributional Generalization:}
{\it 
Let $f$ be a trained classifier.
Then $f$ satisfies Distributional Generalization with respect to tests $\cT$ if:
\begin{align}
\forall T \in \mathcal{T}:
\quad
\E_{\substack{
x \sim \textrm{TrainSet}\\
\hat y \gets f(x)
}}
[
T(x, \hat{y})
]
\approx
\E_{\substack{
x \sim \textrm{TestSet}\\
\hat y \gets f(x)
}}
[
T(x, \hat{y})
]
\end{align}
}
We write this property as $\Dtr \approx^{\cT} \Dte$.
This states that the train and test distribution have
similar expectations for all functions in the family $\mathcal{T}$.
For the singleton set $\cT = \{\Terr\}$, this is equivalent to classical generalization, but it may hold for much larger sets $\cT$.
For example in Experiment \ref{exp:intro1}, the train and test distributions match with respect to the test function ``\emph{Fraction of true cats labeled as object}.'' 
In fact, we find that the family $\cT$ is so large in practice that it is best to think of Distributional Generalization
as stating that the distributions $\Dtr$ and $\Dte$
are close \emph{as distributions}.

This property becomes especially interesting 
for interpolating classifiers, 
which fit their train sets exactly.
Here, the Train Distribution $(x_i, f(x_i))$
is exactly equal\footnote{The formal definition of Train Distribution, in Section \ref{sec:general},
includes the randomness of sampling the train
set as well. We consider a fixed train set in the Introduction for sake of exposition.} to the original distribution $(x, y) \sim \cD$,
since $f(x_i) = y_i$ on the train set.
In this case, distributional generalization claims that the
output distribution $(x, f(x))$ of the model on test samples
is close to the \emph{true} distribution $(x, y)$.
The following conjecture specializes Distributional Generalization
to interpolating classifiers, and will be the main focus of our work.

{\bf Interpolating Indistinguishability Meta-Conjecture (informal):}
{\it 
For interpolating classifiers $f$, and a large family $\mathcal{T}$ of
test functions, the distributions:
\begin{align}
\label{eqn:metaconj}
\boxed{
(x, f(x))_{x \in \TestSet}
~\approx^{\mathcal{T}}~
(x, f(x))_{x \in \TrainSet}
~\equiv~
(x, y)_{x, y \sim \cD}
}
\end{align}
}

This is a ``meta-conjecture'', which becomes a concrete conjecture
once the family of tests $\mathcal{T}$ is specified.
One of the main contributions of our work
is formally stating two concrete instances of this conjecture---
specifying exactly the family of tests $\mathcal{T}$
and their dependence on problem parameters 
(the distribution, model family, training procedure, etc).
It captures behaviors far more general than Experiment 1,
and empirically applies 
across a variety of natural settings
in machine learning.

\subsection{Summary of Contributions}
We extend the classical framework of generalization
by introducing Distributional Generalization, in which
the train and test behavior of models are close \emph{as distributions}.
Informally, for trained classifiers $f$,
its outputs on the train set
$(x, f(x))_{x \in \TrainSet}$
are close in distribution to its outputs on the test set
$(x, f(x))_{x \in \TestSet}$,
where the form of this closeness depends on specifics of the model, training procedure, and distribution.
This notion is more fine-grained than classical generalization, since it
considers the entire distribution of model outputs instead of just the test error.

We initiate the study of Distributional Generalization
across various domains in machine learning.
For interpolating classifiers, we state two formal conjectures which predict the form of distributional closeness that can be expected for a given model and task:
\begin{enumerate}
    \item {\bf Feature Calibration Conjecture} (Section~\ref{sec:dist-features}):
    Interpolating classifiers, when trained on samples from a distribution, 
    will match this distribution up to all ``distinguishable features'' (Definition~\ref{def:dist-feature-L}).
    \item {\bf Agreement Conjecture} (Section~\ref{sec:agree}):
    For two interpolating classifiers of the same type,
    trained independently on the same distribution,
    their \emph{agreement probability} with each other on test samples
    roughly matches their \emph{test accuracy}.
\end{enumerate}
We perform a number of experiments surrounding these conjectures,
which reveal new behaviors of standard interpolating classifiers
(e.g. ResNets, MLPs, kernels, decision trees).
We prove our conjectures for 1-Nearest-Neighbors (Theorem~\ref{thm:Ltest}),
which suggests some form of ``locality'' as the underlying mechanism.
Finally, we discuss extending these results to non-interpolating methods in Section~\ref{sec:gg}.
Our experiments and conjectures shed new light on the structure of interpolating classifiers, which are extensively studied in recent years yet still poorly understood.

\section{Related Work}

Our work is inspired by the broader study of interpolating and overparameterized methods
in machine learning; a partial list of works in this theme includes \citet{zhang2016understanding,belkin2018overfitting, belkin2018understand, belkin2019reconciling, liang2018just, nakkiranDeep, mei2019generalization,schapire1998boosting, breiman1995reflections,ghorbani2019linearized,hastie2019surprises,bartlett2020benign,advani2017high,geiger2019jamming,gerace2020generalisation,chizat2020implicit,goldt2019generalisation, arora2019fine,allen2019learning,neyshabur2018towards,dziugaite2017computing,muthukumar2020harmless,neal2018modern,ji2019polylogarithmic,soudry2018implicit}.

{\bf Interpolating Methods.}
Many of the best-performing techniques on high-dimensional tasks are interpolating methods,
which fit their train samples to 0 train error.
This includes neural networks and kernels on images \citep{he2016deep, shankar2020neural}, and random forests on tabular data \citep{fernandez2014we}.
Interpolating methods have been extensively studied both recently
and in the past, since we do not 
theoretically understand their practical success
\citep{schapire1998boosting, schapire1999theoretical, breiman1995reflections, zhang2016understanding,belkin2018overfitting, belkin2018understand, belkin2019reconciling, liang2018just, mei2019generalization, hastie2019surprises, nakkiranDeep}.
In particular, much of the classical work in statistical learning theory 
(uniform convergence, VC-dimension, Rademacher complexity, regularization, stability)
fails to explain the success of
interpolating methods \citep{zhang2016understanding, belkin2018overfitting, belkin2018understand, nagarajan2019uniform}.
The few techniques which do apply to interpolating methods (e.g. margin theory~\citep{schapire1998boosting})
remain vacuous on modern neural networks and kernels.

\paragraph{Learning Substructures.}
The observation that powerful networks can pick up on
finer aspects of the distribution than their training labels reveal also exists in various forms in the literature.
For example, \citet{gilboa2019wider} found 
that standard CIFAR-10 networks
tend to cluster left-facing and right-facing horses separately
in activation space, when visualized via activation atlases~\citep{carter2019activation}.
Such fine-structural aspects of distributions can also be seen at the level of individual neurons (e.g. \citet{cammarata2020thread:,radford2017learning,olah2018building,zhou2014object,bau2020understanding}).

{\bf Decision Trees.}
In a similar vein to our work, \citet{wyner2017explaining, olson2018making} investigate decision trees,
and show that random forests are equivalent to 
a Nadaraya–Watson smoother \cite{nadaraya1964estimating, watson1964smooth} with a certain smoothing kernel.
Decision trees~\citep{breiman1984classification} are often intuitively thought of as 
``adaptive nearest-neighbors,''
since they are explicitly a spatial-partitioning method
\citep{hastie2009elements}.
Thus, it may not be surprising that decision trees behave similarly
to 1-Nearest-Neighbors.
\citet{wyner2017explaining, olson2018making} took steps towards characterizing and understanding this behavior --
in particular, \citet{olson2018making}
defines an equivalent smoothing kernel corresponding to a random forest,
and empirically investigates the quality of the conditional density estimate.
Our work presents a formal characterization of the quality of this conditional density estimate
(Conjecture~\ref{conj:approx}), which is a novel characterization even for decision trees, as far as we know.

{\bf Kernel Smoothing.}
The term kernel regression is sometimes used in the literature to
refer to kernel \emph{smoothers}, 
such as the Nadaraya–Watson kernel smoother \citep{nadaraya1964estimating, watson1964smooth}.
But in this work we use the term ``kernel regression'' to refer only to
regression in a Reproducing Kernel Hilbert Space, as described in the
experimental details.

{\bf Label Noise.}
Our conjectures also describe the behavior of neural networks under label noise,
which has been empirically and theoretically studied in the past,
though not formally characterized before
\citep{zhang2016understanding, belkin2018understand, rolnick2017deep, natarajan2013learning, thulasidasan2019combating, ziyin2020learning, chatterji2020finite}.
Prior works have noticed that vanilla interpolating networks are sensitive to label noise
(e.g. Figure 1 in ~\citet{zhang2016understanding}, and~\citet{belkin2018understand}),
and there are many works on making networks more robust to label noise
via modifications to the training procedure or objective
\citep{rolnick2017deep, natarajan2013learning, thulasidasan2019combating, ziyin2020learning}.
In contrast, we claim this sensitivity to label noise is not necessarily
a problem to be fixed,
but rather a consequence of a stronger property: distributional generalization.

{\bf Conditional Density Estimation.}
Our density calibration property is similar to
the guarantees of a conditional density estimator.
More specifically, Conjecture~\ref{conj:approx}
states that an interpolating classifier \emph{samples}
from a distribution approximating the conditional density
of $p(y | x)$ in a certain sense.
Conditional density estimation has been well-studied
in classical nonparametric statistics (e.g. the Nadaraya–Watson kernel smoother \citep{nadaraya1964estimating, watson1964smooth}).
However, these classical methods behave poorly
in high-dimensions, both in theory and in practice.
There are some attempts to extend these classical methods
to modern high-dimentional problems 
via augmenting estimators with neural networks
(e.g. \citet{rothfuss2019conditional}).
Random forests have also been known to exhibit
properties similar to conditional density estimators.
This has been formalized in various ways,
often only with asymptotic guarantees \citep{meinshausen2006quantile,pospisil2018rfcde, athey2019generalized}.

No prior work that we are aware of
attempts to characterize the quality of the resulting density estimate
via testable assumptions, as we do with our formulation of Conjecture~\ref{conj:approx}.
Finally, our motivation is not to design good conditional density estimators, but rather to study properties of interpolating classifiers --- which we find happen to share properties of density estimators.

{\bf Uncertainty and Calibration.}
The Agreement Property (Conjecture~\ref{claim:agree})
bears some resemblance to uncertainty
estimation (e.g. \citet{lakshminarayanan2017simple}),
since it estimates the the test error of a classifier using
an ensemble of 2 models trained on disjoint train sets.
However, there are important caveats: (1) Our Agreement Property
only holds on-distribution, and degrades on off-distribution inputs.
Thus, it is not as helpful to estimate out-of-distribution errors.
(2) It only gives an estimate of the average test error, and does not imply pointwise calibration estimates for each sample.

Feature Calibration (Conjecture~\ref{conj:approx}) is also related to the concepts of calibration
and multicalibration~\citep{guo2017calibration, niculescu2005predicting, hebert2018multicalibration}.
In our framework, calibration is implied by Feature Calibration for a specific set of partitions $L$
(determined by level sets of the classifier's confidence).
However, we are not concerned with a specific set of partitions
(or ``subgroups'' in the algorithmic fairness literature)
but we generally aim to characterize for which partitions Feature Calibration holds.
Moreover, we consider only hard-classification decisions and not confidences,
and we study only standard learning algorithms which are not given any distinguished set of subgroups/partitions in advance.
Our notion of distributional generalization is also related
to the notion of ``distributional subgroup overfitting'' introduced
recently by \citet{yaghini2019disparate} to study algorithmic fairness.
This can be seen as studying distributional generalization for a specific family of tests
(determined by distinguished subgroups in the population).

{\bf Locality and Manifold Learning.}
Our intuition for the behaviors in this work is that they arise due to some form of ``locality'' of the trained classifiers, in an appropriate space.
This intuition is present in various forms in the literature,
for example: the so-called called ``manifold hypothesis,''
that natural data lie on a low-dimensional manifold (e.g. \citet{narayanan2010sample, sharma2020neural}),
as well as works on local stiffness of the loss landscape \citep{fort2019stiffness},
and works showing that overparameterized neural networks
can learn hidden low-dimensional structure in high-dimensional settings
\citep{gerace2020generalisation, bach2017breaking, chizat2020implicit}.
It is open to more formally understand connections between our work and the above.

\section{Preliminaries}

{\bf Notation.}
We consider joint distributions $\cD$
on $x \in \cX$ and discrete $y \in \cY = [k]$.
Let $\cD^n$ denote $n$ iid samples from $\cD$ and $S = \{(x_i, y_i)\}$ denote a train set.
Let $\cF$ denote the training procedure of a classifier family (including architecture and training algorithm),
and let $f \gets \Train_{\cF}(S)$ denote training a classifier $f$ on train set $S$.
We consider classifiers which output hard decisions $f: \cX \to \cY$.
Let $\NN_{S}(x) = x_i$ denote the nearest neighbor to $x$ in train-set $S$,
with respect to a distance metric $d$.
Our theorems will apply to any distance metric, and so we leave this unspecified.
Let $\NNf_S(x)$ denote the nearest neighbor estimator itself,
that is, $\NNf_S(x) := y_i$ where $x_i = \NN_S(x)$.

{\bf Experimental Setup.} Full experimental details are provided in Appendix~\ref{app:experiment}.
Briefly, we train all classifiers to interpolation unless otherwise specified--- that is, to 0 train error.
We use standard-practice training techniques for all methods with minor hyperparameter modifications for training to interpolation.
In all experiments, we consider only the hard-classification decisions, and not e.g. the softmax probabilities.
Neural networks (MLPs and ResNets~\citep{he2016deep}) are trained with Stochastic Gradient Descent.
Interpolating decision trees are trained using the growth rule from Random Forests~\citep{breiman2001random},
growing until all leafs have a single sample.
For kernel classification, we consider both kernel regression on one-hot labels and kernel SVM,
with small or $0$ values of regularization (which is often optimal, as in~\citet{shankar2020neural}).
Section~\ref{sec:gg} considers non-interpolating versions of the above methods
(via early-stopping or regularization).

\subsection{Distributional Closeness}
\label{sec:dist-closeness}
For two distributions $P, Q$ over $\cX \x \cY$,
let $P \approx_\eps Q$
denote $\eps$-closeness in total variation distance; that is,
$TV(P, Q) = \frac{1}{2}||P-Q||_1 \leq \eps$.
Recall that TV-distance has an equivalent variational characterization:
For distributions $P, Q$ over $\cX \x \cY$, we have
\[
TV(P, Q) =
\sup_{T: \cX \x \cY \to [0, 1]}
\left|\E_{(x, y) \sim P} [T(x, y)]
-
\E_{(x, y) \sim Q} [T(x, y)]\right|
\]
A ``test'' (or ``distinguisher'') here is a function
$T: \mathcal{X} \x \mathcal{Y} \to [0, 1]$
which accepts a sample from either distribution,
and is intended to classify the sample as either
from distribution $P$ or $Q$.
TV distance is then the advantage of the best distinguisher among all bounded tests.
More generally, for any family
$\cT \subseteq \{ T: \cX \x \cY \to [0, 1]\}$ of tests,
we say distributions $P$ and $Q$ are ``$\eps$-indistinguishable up to $\cT$-tests'' if they are
close with respect to all tests in class $\cT$.
That is, 
\begin{align}
\label{eqn:test}
P \approx_\eps^{\cT} Q
\iff 
\quad
\sup_{T \in \cT}
\left|
\E_{(x, y) \sim P}
[T(x, y)]
-
\E_{(x, y) \sim Q}
[T(x, y)]
\right|
\leq \eps
\end{align}
This notion of closeness is also known as an Integral Probability Metric~\citep{muller1997integral}.
Throughout this work, we will define specific families of distinguishers $\cT$
to characterize the sense in which the output distribution $(x, f(x))$
of classifiers is close to their input distribution $(x, y) \sim \cD$.
When we write $P \approx Q$, we are making an informal claim
in which we mean $P \approx_\eps Q$ for some small but unspecified $\eps$.

\subsection{Framework for Indistinguishability}
\label{sec:general}

Here we setup the formal objects studied in the remainder of the paper.
This formal description of Train and Test distributions differs slightly
from the informal description
in the Introduction, because we want to study the generalization properties
of an entire end-to-end training procedure ($\Train_\cF$), and not just properties
of a fixed classifier ($f$).
We thus consider the following three distributions over $\cX \x \cY$.

\noindent
\mbox{}%
\hfill%
\fbox{
\begin{minipage}[t]{.3\textwidth}
{
\underline{{\bf Source $\cD$:}} \hspace{0.1in}  {\bf $(x, y)$} \\
}
\noindent
where $x, y \sim \cD$
\vspace{12pt}
\end{minipage}}%
\hfill
\fbox{
\begin{minipage}[t]{.3\textwidth}
\underline{{\bf Train $\Dtr$:}}
{ \bf $(x_{\mathrm{tr}}, f(x_{\mathrm{tr}}))$} \\
$S \sim \cD^n$, $f \gets \Train_{\cF}(S)$,\\
$x_{\mathrm{tr}}, y_{\mathrm{tr}} \sim S$
\end{minipage}}%
\hfill
\fbox{\begin{minipage}[t]{.3\textwidth}
{
\underline{{\bf Test $\Dte$}}
$(x, f(x))$ \\
}
$S \sim \cD^n$, $f \gets \Train_{\cF}(S)$,\\
$x, y \sim \cD$
\end{minipage}}%
\hfill%
\mbox{}%

The {\bf Source Distribution $\cD$} is simply the
original distribution.
To sample from the {\bf Train Distribution $\Dtr$},
we first sample a train set $S \sim \cD^n$,
train a classifier $f$ on it,
then output $(x_{\mathrm{tr}}, f(x_{\mathrm{tr}}))$
for a random \emph{train point} $x_{\mathrm{tr}}$.
That is, $\Dtr$ is the distribution of input and outputs of a trained classifier $f$
on its train set.
To sample from the {\bf Test Distribution $\Dte$},
do we this same procedure,
but output $(x, f(x))$ for a random \emph{test point} $x$.
That is, the $\Dte$ is the distribution of input and outputs of a trained classifier $f$
at test time.
The only difference between the Train Distribution
and Test Distribution is that the point $x$ is sampled
from the train set or the test set, respectively.\footnote{
Technically, these definitions require training a fresh classifier for each sample,
using independent train sets. We use this definition because we believe it is natural,
although for practical reasons most of our experiments train a single classifier $f$
and evaluate it on the entire train/test set.
}

For interpolating classifiers,
$f(x_{\mathrm{tr}}) = y_{\mathrm{tr}}$ on the train set,
and so the Source and Train distributions are equivalent:
\begin{equation}
\textrm{For interpolating classifiers $f$:}
\quad
\cD \equiv \Dtr
\end{equation}

Our general thesis is that the Train and Test Distributions
are indistinguishable under a variety of test families $\cT$.
Formally, we argue that 
for certain families of tests $\cT$
and interpolating classifiers $\cF$,
\begin{equation}
\textrm{Indistinguishability Conjecture:}
\quad
\boxed{
\cD \equiv \Dtr \approx_\eps^{\cT} \Dte
}
\end{equation}

Sections~\ref{sec:dist-features} and~\ref{sec:agree}
give specific families of tests $\cT$ for which 
these distributions are indistinguishable.
The quality of this distributional closeness will depend on details of the classifier family and distribution,
in ways which we will specify.

\section{Feature Calibration}
\label{sec:dist-features}

The distributional closeness of Experiment 1 is subtle,
and depends on the classifier architecture, distribution, and training method. For example, Experiment 1 does not hold if we use a fully-connected network (MLP) instead of a ResNet,
or if we early-stop the ResNet instead of training to interpolation (see Appendix \ref{app:intro-exp-1}). Both these scenarios fail in different ways:
An MLP cannot properly distinguish cats from dogs even when trained on real CIFAR-10 labels,
and so (informally) it has no hope of behaving differently on cats in the setting of Experiment 1.
On the other hand, an early-stopped ResNet for Experiment 1 does not label 30\% of cats as objects on the \emph{train set},
since it does not interpolate, and thus has no hope of reproducing this behavior on the test set.

We now characterize these behaviors, and their dependency on problem parameters,
via a formal conjecture.
This conjecture characterizes a family of tests $\mathcal{T}$ for which the output distribution of a classifier $(x, f(x)) \sim \Dte$
is ``close'' to the source distribution $(x, y) \sim \cD$.
At a high level, we argue that the distributions $\Dte$ and $\cD$
are statistically close if we first ``coarsen'' the domain of $x$ by some
labelling $L: \cX \to [M]$. 
That is, for certain partitions $L$, the following
distributions are statistically close:

\vspace{-5pt}
\begin{align*}
(L(x), f(x)) &\approx_{\eps} (L(x), y)
\end{align*}

We first explain this conjecture (``Feature Calibration'') via a toy example, and then we state the conjecture
formally in Section~\ref{sec:formal}.

\subsection{Toy Example}
\label{sec:toy}
\newcommand{\fpart}[1]{\hyperref[fig:partition]{\ref*{fig:partition}#1}}

Consider a distribution on points $x \in \R^2$ and binary labels $y$, as visualized in
Figure~\fpart{A}.
This distribution consists of four clusters 
\texttt{\{Truck, Ship, Cat, Dog\}} which are labeled either
\texttt{Object} or \texttt{Animal}, depicted in red and blue respectively.
One of these clusters --- the \texttt{Cat} cluster --- is mislabeled as class \texttt{Object} with probability 30\%.
Now suppose we have an interpolating classifier $f$ for this distribution, obtained in some way,
and we wish to quantify the closeness between distributions $(x, y) \approx (x, f(x))$ on the test set.
Figure~\fpart{A} shows test points $x$ along with their test labels $y$ -- these
are samples from the source distribution $\cD$.
Figure~\fpart{B} shows these same test points, but labeled according to $f(x)$ --
these are samples from the test distribution $\Dte$.
The shaded red/blue regions in Figure~\fpart{B} shows the decision boundary of the classifier $f$.

These two distributions do not match exactly -- there are some test points where the true label $y$ and classifier output $f(x)$ disagree.
However, if we ``coarsen'' the domain into the four clusters \texttt{\{Truck, Ship, Cat, Dog\}},
then the marginal distribution of labels within each cluster matches between the classifier outputs $f(x)$
and true labels $y$.
In particular, the fraction of \texttt{Cat} points that are labeled \texttt{Object} is similar between
Figures~\fpart{C} and \fpart{D}.
This is equivalent to saying that the joint distributions $(L(x), y)$ and $(L(x), f(x))$ 
are statistically close, for the partition $L: \cX \to \texttt{\{Truck, Ship, Cat, Dog\}}$.
That is, if we can only see points $x$ through their cluster-label $L(x)$,
then the distributions $(x, y)$ and $(x, f(x))$ will appear close.
These two ``coarsened'' joint distributions are also what we plotted in Figure~\ref{fig:intro}
from the Introduction, where we considered the partition $L(x) := \texttt{CIFAR\_Class}(x)$, the CIFAR-10 class of $x$.

\begin{figure}[t]
\centering
\includegraphics[width=\linewidth]{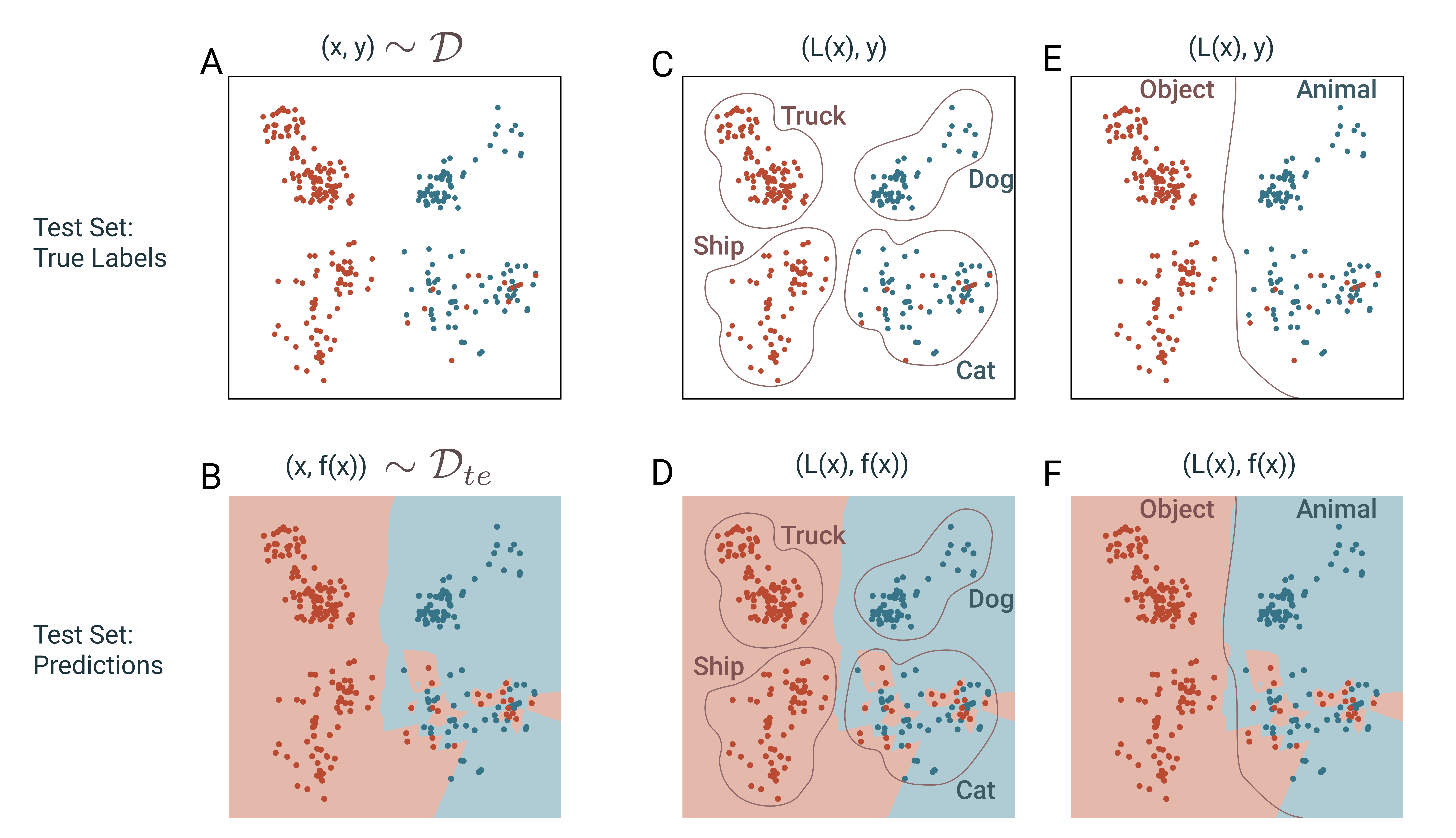}
\caption{{\bf Toy Example: Feature Calibration.}
Schematic of the distributions discussed in Section~\ref{sec:toy},
showing a toy example of the Feature Calibration conjecture for several distinguishable features $L$.
}
\label{fig:partition}
\end{figure}

There may be many such partitions $L$ for which the above distributional closeness holds.
For example, the coarser partition $L_2: \cX \to \texttt{\{Animal, Object\}}$ also
works in our example, as shown in Figures~\fpart{E} and \fpart{F}.
However, clearly not all partitions $L$ will satisfy closeness, since the distributions themselves
are not statistically close. 
For a finer partition which splits the clusters into smaller pieces (e.g. based on the age of the dog),
the distributions may not match unless we use very powerful classifiers or many train samples.
Intuitively, allowable partitions are those which can be
learnt from samples. To formalize the set of allowable partitions $L$,
we define a \emph{distinguishable feature}:
a partition of the domain $\cX$ that is learnable for a given family of models. 
For example, in Experiment \ref{exp:intro1}, the
partition into CIFAR-10 classes would be a distinguishable feature for ResNets, but not for MLPs. We now state the formal definition of a distinguishable partition and the formal conjecture.

\subsection{Formal Definitions}
\label{sec:formal}
We first define a \emph{distinguishable feature}:
a labeling of the domain $\cX$ that is learnable for a given family of models. 
This definition depends on the family of models $\cF$, the distribution $\cD$,
and the number of train samples $n$.
\begin{definition}[$(\eps, \cF, \cD, n)$-Distinguishable Feature]
\label{def:dist-feature-L}
For a distribution $\cD$ over $\cX \x \cY$,
number of samples $n$,
family of models $\cF$, and small $\eps \geq 0$,
an \emph{$(\eps, \cF, \cD, n)$-distinguishable feature}
is a partition $L: \cX \to [M]$
of the domain $\cX$ into $M$ parts, such that
training a model from $\cF$
on $n$ samples labeled by $L$
works to classify $L$ with high test accuracy.

Precisely, $L$ is a distinguishable feature if the following procedure succeeds with probability at least $1-\eps$:
\begin{enumerate}
    \item Sample a train set $S \gets \{(x_i, L(x_i))\}$ of $n$ samples $(x_i, y_i) \sim \cD$,
    labeled by the partition $L$.
    \item Train a classifier $f \gets Train_{\cF}(S)$.
    \item Sample a test point $x \sim \cD$, and check that $f$ correctly classifies its partition:
    Output success iff $f(x) = L(x)$.
\end{enumerate}
That is, $L$ is a $\epsL$-distinguishable feature if:
$$
\Pr_{\substack{
S = \{(x_i, L(x_i)\}_{x_1, \dots, x_n \sim \cD}\\
f \gets \Train_{\cF}(S)\\
x \sim \cD
}}[
f(x) = L(x)
]
\geq 1-\eps
$$
\end{definition}
To recap, this definition is meant to
capture a labeling of the domain $\cX$ that is learnable for a given family of models and training procedure.
Note that this definition only depends on the marginal distribution of $\cD$ on $x$,
and does not depend on the label distribution $p_{\cD}(y | x)$.
The definition of distinguishable feature must depend on the classifier family $\cF$ and number of samples $n$,
since a more powerful classifier can distinguish more features.
Note that there could be many distinguishable features for a given setting $\epsL$ ---
including features not implied by the class label,
such as the presence of grass in a CIFAR-10 image.

Our main conjecture in this section is that
the test distribution $(x, f(x)) \sim \Dte$
is statistically close to the source distribution $(x, y) \sim \cD$
when the domain is ``coarsened'' by a distinguishable feature.
That is, the distributions $(L(x), f(x))$ and $(L(x), y)$
are \emph{statistically} close for all distinguishable features $L$.
Formally:

\begin{conjecture}[Feature Calibration]
\label{conj:approx}
For all natural distributions $\cD$,
number of samples $n$,
family of interpolating models $\cF$,
and $\eps \geq 0$,
the following distributions are statistically close for all $\epsL$-distinguishable features $L$:
\begin{align}
\underset{
\substack{
f \gets \Train_{\cF}(\cD^n)\\
x, y \sim \cD
}
}{
(L(x), f(x))
}
\quad\approx_\eps\quad
\underset{
x, y \sim \cD
}{
(L(x), y)
}
\end{align}
\end{conjecture}

Notably, this holds \emph{for all} distinguishable features $L$,
and it holds ``automatically'' -- we simply train a classifier, without specifying any particular partition.
The statistical closeness predicted is within $\eps$, which is determined by the $\eps$-distinguishability
of $L$ (we usually think of $\eps$ as small).
As a trivial instance of the conjecture,
suppose we have a distribution with deterministic labels,
and consider the $\eps$-distinguishable feature $L(x) := y(x)$, i.e. the label itself.
The $\eps$ here is then simply the test error of $f$,
and Conjecture~\ref{conj:approx} is true by definition.
The formal statements of Definition~\ref{def:dist-feature-L} and Conjecture~\ref{conj:approx}
may seem somewhat arbitrary, involving many quantifiers over $\epsL$.
However, we believe these statements are natural.
To support this, in Section~\ref{sec:1nn} we prove that Conjecture~\ref{conj:approx} is formally
true as stated for 1-Nearest-Neighbor classifiers.

{\bf Connection to Indistinguishability.}
Conjecture~\ref{conj:approx} can be equivalently phrased as an instantiation of our general
Indistinguishably Conjecture:
the source distribution $\cD$ and 
test distribution $\Dte$ are
``indistinguishable up to $L$-tests''.
That is, Conjecture~\ref{conj:approx}
is equivalent to the statement 
\begin{align}
\Dte \approx_\eps^{\cL} \cD
\end{align}
where $\cL$ is the family of all tests which depend on $x$
only via a distinguishable feature $L$. That is,
$
\cL := \{(x, y) \mapsto T(L(x), y) :
(\eps, \cF, \cD, n)\text{-distinguishable feature }L
\text{ and } T: [M] \x \cY \to [0, 1] \}
$. In other words, $\Dte$ is indistinguishable from $\cD$
to any distinguisher that only sees the input $x$ via a distinguishable feature $L(x)$.

\subsection{Experiments}
We now empirically validate our conjecture in a variety of settings in machine learning,
including neural networks, kernel machines, and decision trees.
To do so, we begin by considering the simplest possible distinguishable feature,
and progressively consider more complex ones. 
Each of the experimental settings below highlights a different aspect of interpolating classifiers,
which may be of independent theoretical or practical interest.
We summarize the experiments here; detailed descriptions are provided in Appendix~\ref{app:density}.

\textbf{Constant Partition:} Consider the trivially-distinguishable \emph{constant} feature $L(x) = 0$.
Then, Conjecture~\ref{conj:approx} states that the marginal distribution
of class labels for any interpolating classifier $f(x)$
is close to the true marginals $p(y)$.
That is, irrespective of the classifier's test accuracy, it outputs the ``right'' proportion of class labels on the test set, even when there is strong class imbalance.

To show this, we construct a dataset based on CIFAR-10 that has class imbalance. For class $k \in \{0...9\}$, sample $(k+1)\times 500$ images from that class. This will give us a dataset where classes will have marginal distribution $p(y = \ell) \propto \ell+1$ for classes $\ell \in [10]$, as shown in Figure \ref{fig:constant-L}. We do this both for the training set and the test set, to keep the distribution $\cD$ fixed.
We then train a variety of classifiers (MLPs, Kernels, ResNets) to interpolation on this dataset,
which have varying levels of test errors (9-41\%).
The class balance of classifier outputs on the (rebalanced) test set
is then close to the class balance on the train set, even for poorly generalizing classifiers. Full experimental details and results are described in Appendix \ref{app:density}.
Note that a 1-nearest neighbors classifier would have this property.

\begin{figure}[t]
    \includegraphics[width=\linewidth]{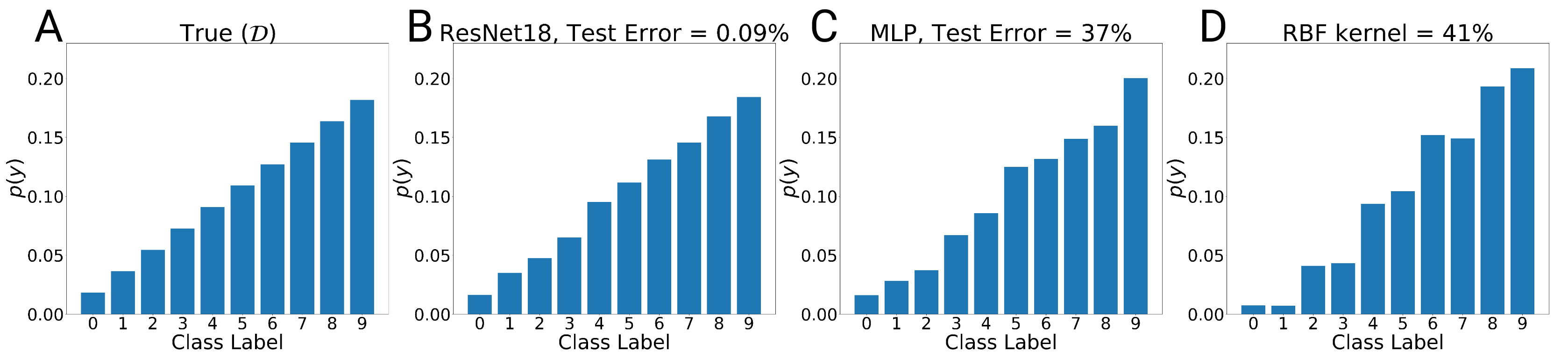}
    \centering
    \caption{{\bf Feature Calibration for Constant Partition $L$: }
     The CIFAR-10 train and test sets are class rebalanced according to (A).
     Interpolating classifiers are trained on the train set,
     and we plot the class balance of their outputs on the test set.
     This roughly matches the class balance of the train set, even for poorly-generalizing classifiers.}
    \label{fig:constant-L}
\end{figure}

{\bf Class Partition:} We now consider settings (datasets and models)
where the original class labels are a distinguishable feature.
For instance, the CIFAR-10 classes are distinguishable by ResNets,
and MNIST classes are distinguishable by the RBF kernel. Since the conjecture holds for any arbitrary label distribution $p(y|x)$, we consider many such label distributions and show that, for instance, the joint distributions $(\textrm{Class}(x), y)$ and $(\textrm{Class}(x), f(x))$ are close. This includes the setting of Experiments 1 and 2 from the Introduction. 

\begin{figure}[p]
    \includegraphics[width=\linewidth]{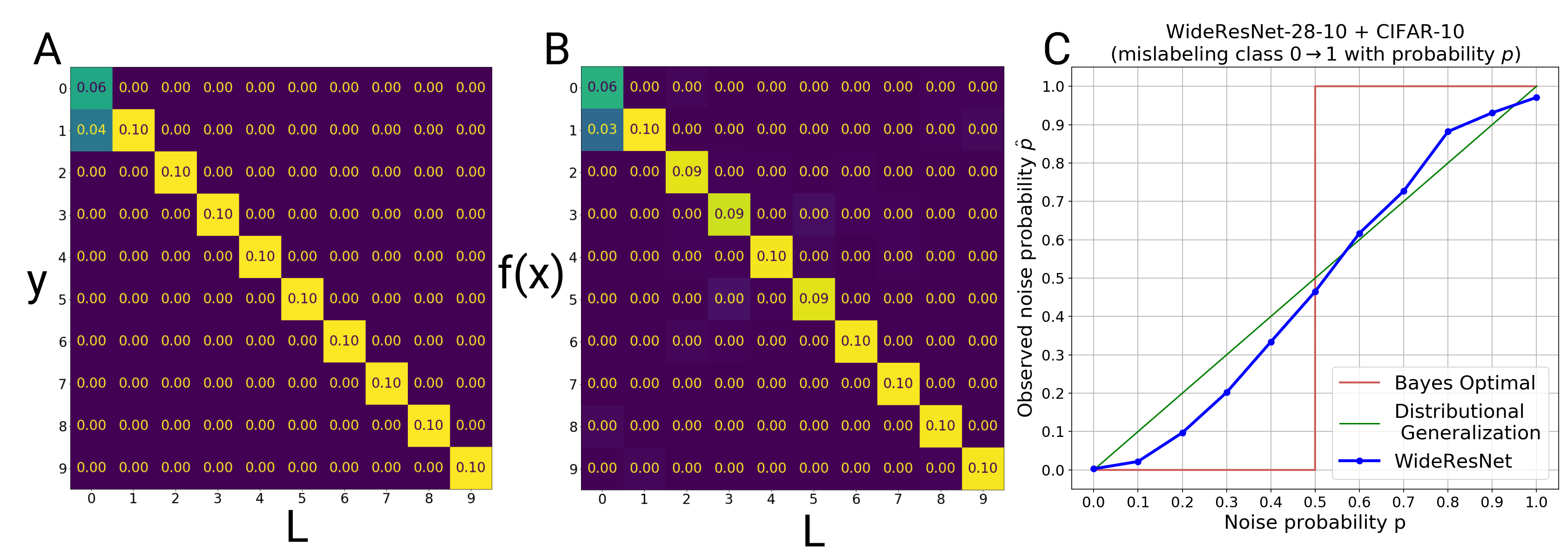}
    \centering
    \caption{\textbf{Feature Calibration with original classes on CIFAR-10}: We train a WRN-28-10 on the CIFAR-10 dataset where we mislabel class $0 \rightarrow 1$ with probability $p$.
    (A): Joint density of the distinguishable features $L$ (the original CIFAR-10 class) and the classification task labels $y$ on the train set for noise probability $p=0.4$.
    (B): Joint density of the original CIFAR-10 classes $L$ and the network outputs $f(x)$ on the test set.
    (C): Observed noise probability in the network outputs on the test set (the (1, 0) entry of the matrix in B) for varying noise probabilities $p$}
    \label{fig:target-varyp}
\end{figure}

\begin{figure}[p]
\centering
\includegraphics[width=0.7\linewidth]{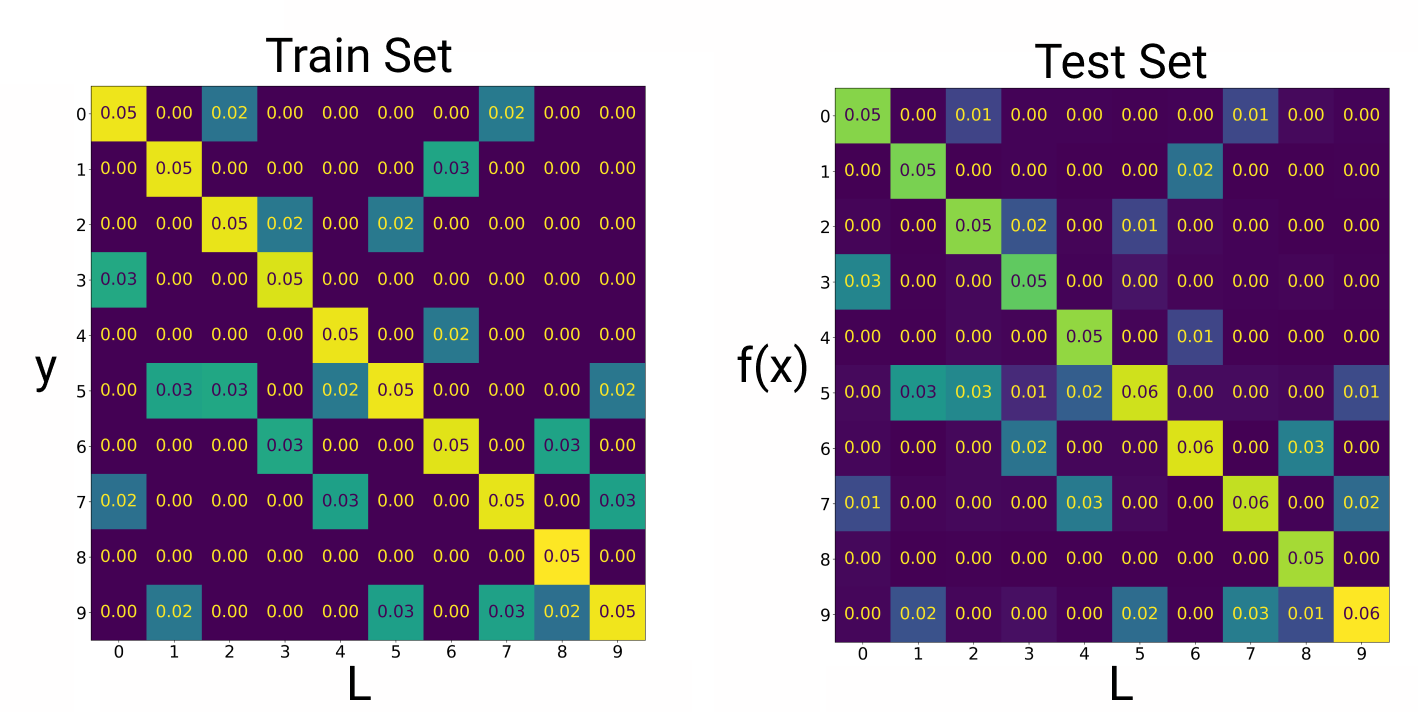}
\caption{{\bf Feature Calibration with random confusion matrix on CIFAR-10: }
Left: Joint density of labels $y$ and original class $L$ on the train set.
Right: Joint density of classifier predictions $f(x)$ and original class $L$
on the test set,
for a WideResNet28-10 trained to interpolation. These two joint densities are close,
as predicted by Conjecture~\ref{conj:approx}.}
\label{fig:wrn_random}
\end{figure}

\begin{figure}[p]
    \includegraphics[width=\linewidth]{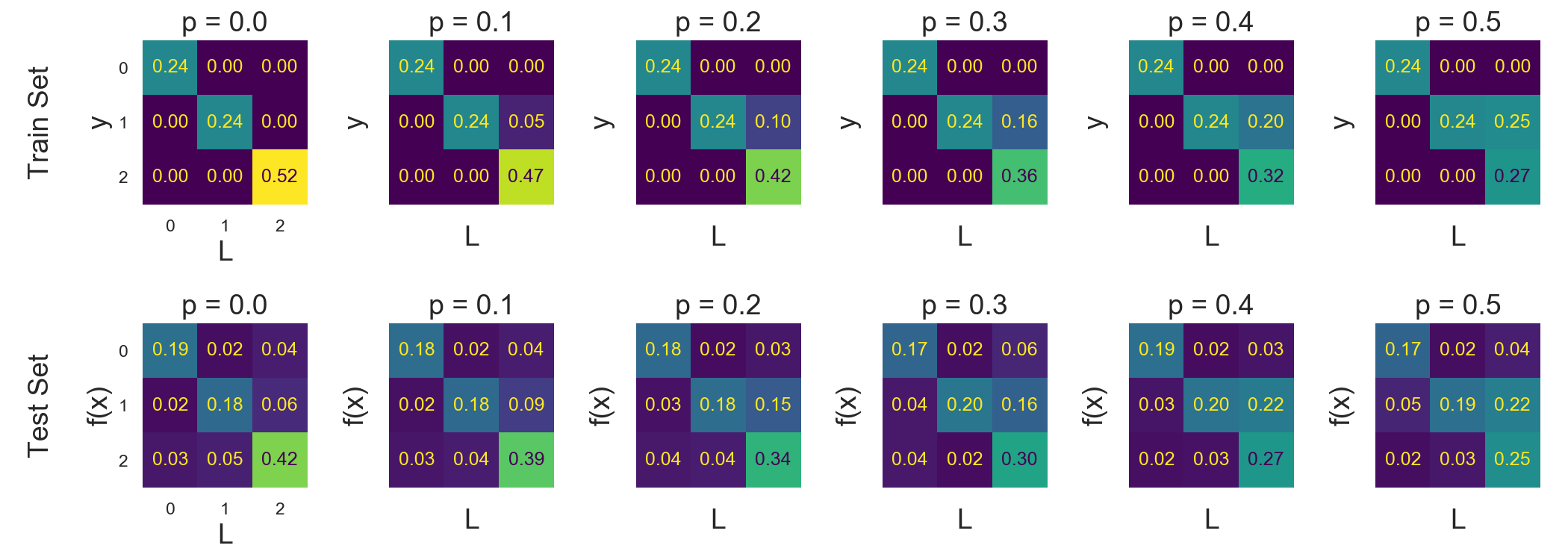}
    \centering
    \caption{{\bf Feature Calibration for Decision trees on UCI (molecular biology).}
    We add label noise that takes class $2$ to class $1$ with probability $p \in [0, 0.5]$.
    The top row shows the confusion matrix of the true class $L(x)$ vs. the label $y$
    on the train set, for varying levels of noise $p$.
    The bottom row shows the corresponding confusion matrices of the classifier predictions $f(x)$
    on the test set, which closely matches the train set, as predicted by Conjecture~\ref{conj:approx}.
    }
    \label{fig:dtdna}
\end{figure}

In Figure \ref{fig:target-varyp}, we mislabel class $0 \rightarrow 1$ with probability $p$ in the CIFAR-10 train set. This gives us the joint distribution shown in Figure \ref{fig:target-varyp}A. 
We then train a WideResNet-28-10 on this noisy distribution.
Figure \ref{fig:target-varyp}B shows the joint distribution on the test set.
Figure \ref{fig:target-varyp}C shows the $(1, 0)$ entry of this matrix as we vary $p \in [0, 1]$.
The Bayes optimal classifier for this distribution would behave as a step function (shown in red),
and a classifier that obeys Conjecture 1 exactly would follow the diagonal (in green).
The actual experiment (in blue) is close to the behavior predicted by Conjecture 1.

In fact, our conjecture holds even for a joint density determined by a random confusion matrix on CIFAR-10.
In Figure~\ref{fig:wrn_random}, we first generate a random sparse
confusion matrix on 10 classes, such that each class is preserved with probability 50\%
and flipped to one of two other classes with probability 20\% and 30\% respectively.
We then apply label noise with this confusion matrix to the train set,
and measure the confusion matrix of the trained classifier on the test set.
As expected, the train and test confusion matrices are close, and share the same sparsity pattern.

Figure~\ref{fig:dtdna} shows a version of this experiment for decision trees on
the molecular biology UCI task.
The molecular biology task is a 3-way classification problem: to classify the type of 
a DNA splice junction (donor, acceptor, or neither), given the sequence of DNA (60 bases) surrounding the junction.
We add varying amounts of label noise that flips class 2 to class 1
with a certain probability, and we observe that interpolating decision trees reproduce
this same structured label noise on the test set.
We also demonstrate similar experiments with the Gaussian kernel on MNIST (Figure \ref{fig:kernelgg}),
and several other UCI tasks (Appendix~\ref{app:density}).

{\bf Multiple features:} We now consider a setting where we may have many distinguishable features for a single classification task. The conjecture states that the network should be automatically calibrated for all distinguishable features, even when it is not explicitly provided any information about these features.  For this, we use the CelebA dataset~\citep{liu2015faceattributes},
which contains images of celebrities with various labelled binary attributes per-image (``male'', ``blond hair'', etc). Some of these attributes form a distinguishable feature for ResNet50 as they are learnable to high accuracy \citep{jahandideh2018physical}. We pick one of the hard attributes as the target classification task, where a ResNet-50 achieves 80\% accuracy.  Then we confirm that the output distribution is calibrated with respect to the attributes that form distinguishable features. In this setting, the label distribution is deterministic, and not directly dependent on the distinguishable features, unlike the experiments considered before. Yet, as we see in Figure~\ref{fig:celeba_wrap}, the classifier outputs are correctly calibrated for each attribute. Full details of the experiment are described in Appendix \ref{app:celeba}.

\begin{figure}
\centering
\includegraphics[width=0.6\linewidth]{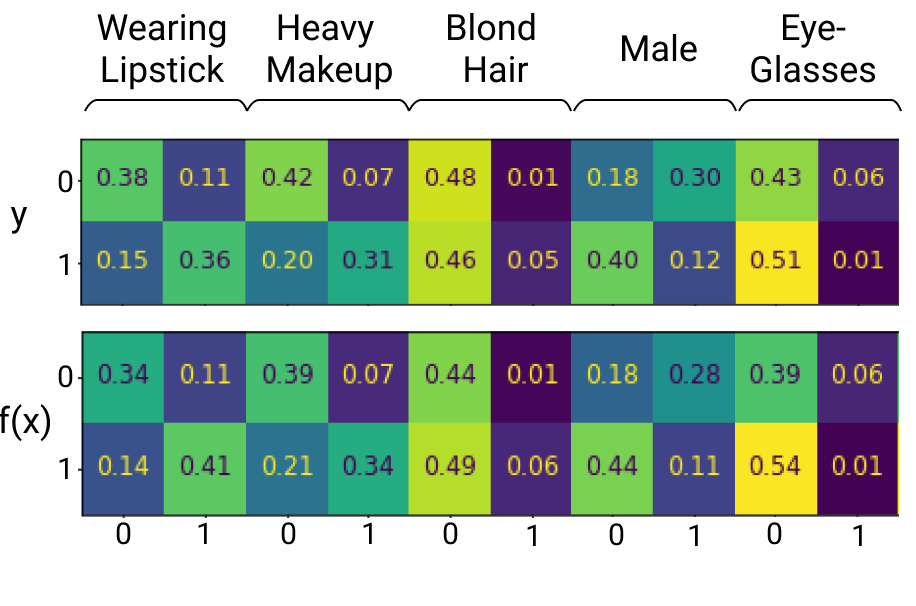}
\caption{\textbf{Feature Calibration for multiple features on CelebA}: We train a ResNet-50 to perform binary classification task on the CelebA dataset. The top row shows the joint distribution of this task label with various other attributes in the dataset. The bottom row shows the same joint distribution for the ResNet-50 outputs on the test set. Note that the network was not given any explicit inputs about these attributes during training.}
\label{fig:celeba_wrap}
\end{figure}

\textbf{Coarse Partition:}
Consider AlexNet trained on ImageNet ILSVRC-2012 \citep{ILSVRC15},
a 1000-class image classification problem including 116 varieties of dogs.
The network only achieves 56.5\% accuracy on the test set, but
it has higher accuracy on coarser label partitions:
for example, it will at least classify most dogs as dogs (with 98.4\% accuracy), though it may mistake the specific dog variety.
In this example, $L(x) \in \{\text{dog, not-dog}\}$ is the distinguishable feature.
Moreover, the network is \emph{calibrated} with respect to dogs:
22.4\% of all dogs in ImageNet are Terriers,
and indeed, the network classifies 20.9\% of all dogs as Terriers
(though it has 9\% error in which specific dogs it classifies as Terriers).
We include similar experiments with ResNets and kernels in Appendix~\ref{app:density}.

\begin{table}
\centering
\begin{tabular}{lrr}
\toprule
Model &  AlexNet &  ResNet50 \\
\midrule
ImageNet accuracy                            &    {\bf 0.565} &     0.761 \\
Accuracy on terriers                         &    0.572 &     0.775 \\
Accuracy for binary \{dog/not-dog\}            &    0.984 &     0.996 \\
Accuracy on \{terrier/not-terrier\} among dogs &    0.913 &     0.969 \\
\midrule
Fraction of real-terriers among dogs         &    {\bf 0.224} &     0.224 \\
Fraction of predicted-terriers among dogs    &    {\bf 0.209} &     0.229 \\
\bottomrule
\end{tabular}
\caption{{\bf Feature Calibration on ImageNet: }
ImageNet classifiers are calibrated with respect to dogs.
For example, all classifiers predict terrier for roughly $\sim22\%$ of all dogs (last row),
though they may mistake which specific dogs are terriers.
See Table~\ref{tab:imagenet} in the Appendix for more models.}
\label{tab:imagenet-trunc}
\end{table}

\subsection{Discussion}
Conjecture~\ref{conj:approx} claims that $\Dte$
is close to $\cD$ up to all tests which are
\emph{themselves learnable}.
That is, if an interpolating method
is capable of learning a certain partition of the domain,
then it will also produce outputs that are calibrated
with respect to this partition, when trained on any problem.
This conjecture thus gives a way of quantifying the resolution
with which classifiers approximate the source distribution $\cD$, 
via properties of the classification algorithm itself.
This is in contrast to many classical ways of
quantifying the approximation of density estimators,
which rely on \emph{analytic} (rather than \emph{operational})
distributional assumptions~\citep{tsybakov2008introduction, wasserman2006all}.

{\bf Proper Scoring Rules.}
If the loss function used in training is a \emph{strictly-proper scoring rule}
such as cross-entropy~\citep{gneiting2007strictly},
then we may expect that in the limit of a large-capacity network
and infinite data, training on samples $\{(x_i, y_i)\}$
will yield a good density estimate of $p(y | x)$ at the softmax layer.
However, this is not what is happening in our experiments:
First, our experiments consider the hard-decisions, not the softmax outputs.
Second, we observe Conjecture~\ref{conj:approx}
even in settings without proper scoring rules
(e.g. kernel SVM and decision trees).

\subsection{1-Nearest-Neighbors Connection}
\label{sec:1nn}
Here we show that the 1-nearest neighbor classifier provably satisfies
Conjecture~\ref{conj:approx}, under mild assumptions.
This is trivially true when the number of train points $n \to \infty$,
such that the train points pack the domain.
However, we do not require any such assumptions:
the theorem below applies generically to a wide class of distributions,
with no assumptions on the ambient dimension of inputs,
the underlying metric, or smoothness of the source distribution.
All the distributional requirements are captured by the preconditions of
Conjecture~\ref{conj:approx}, which require that the feature $L$
is $\eps$-distinguishable to 1-Nearest-Neighbors.
The only further assumption is a weak regularity condition:
sampling the nearest neighbor train point to a random test point
should yield (close to) a uniformly random test point.
In the following, $\NN_S(x)$ refers to the nearest neighbor of point $x$ among points in set $S$.

\begin{theorem}
\label{thm:Ltest}
Let $\cD$ be a distribution over $\cX \x \cY$, and let $n \in \N$ be the number of train samples.
Assume the following regularity condition holds:
Sampling the nearest neighbor train point to a random test point
yields (close to) a uniformly random test point.
That is, suppose that for some small $\delta \geq 0$,
\begin{align}
\{\NN_S(x)\}_{\substack{
S \sim \cD^n\\
x \sim \cD
}}
\quad\approx_\delta\quad
\{x\}_{\substack{
x \sim \cD
}}
\end{align}
Then, Conjecture~\ref{conj:approx} holds. For all $(\eps, \NN, \cD, n)$-distinguishable partitions $L$,
the following distributions are statistically close:
\begin{align}
\{(y, L(x))\}_{x, y \sim \cD}
\quad\approx_{\eps + \delta}\quad
\{(\NNf_S(x), L(x)\}_{\substack{
S \sim \cD^n\\
x, y \sim \cD
}}
\end{align}
\end{theorem}
The proof of Theorem~\ref{thm:Ltest} is straightforward, and provided in Appendix~\ref{sec:proofs}.
We view this theorem both as support for our formalism of Conjecture~\ref{conj:approx},
and as evidence that the classifiers we consider in this work have \emph{local} properties
similar to 1-Nearest-Neighbors.

Note that Theorem~\ref{thm:Ltest} does not hold for the k-nearest neighbor classifier (k-NN),
which takes the plurality vote of K neighboring train points.
However, it is somewhat more general than 1-NN: for example, it holds for a randomized version of k-NN which,
instead of taking the plurality,
randomly picks one of the K neighboring train points (potentially weighted) for the test classification.

\subsection{Pointwise Density Estimation}
\label{sec:ptwise}
In fact, we could hope for an even stronger property than
Conjecture~\ref{conj:approx}.
Consider the familiar example:
we mislabel 20\% of dogs as cats in the CIFAR-10 training data,
and train an interpolating ResNet on this train set.
Conjecture~\ref{conj:approx} predicts that, \emph{on average}
over all test dogs, roughly 20\% of them are classified as cats.
In fact, we may expect this to hold pointwise for each dog:
For a single test dog $x$, if we train a new classifier $f$ (on fresh iid samples from the noisy distribution),
then $f(x)$ will be cat roughly 20\% of the time.
That is, for each test point $x$, taking an ensemble over independent train sets yields an estimate of the conditional density $p(y | x)$.
Informally:
\begin{align}
\text{With high probability over test } x \sim \cD:
\quad
\Pr_{f \gets \Train_{\cF}(\cD^n)}[f(x) = \ell] ~\approx~ p(y = \ell | x)
\end{align}
where the probability on the LHS is over the random sampling of train set, and any randomness in the training procedure.
This behavior would be stronger than,
and not implied by, Conjecture~\ref{conj:approx}.
We give preliminary experiments supporting such a pointwise property in Appendix~\ref{app:ptwise}.

\section{Agreement Property}
\label{sec:agree}

We now present an ``agreement property''
of various classifiers.
This property is independent of the previous section,
though both are instantiations
of our general indistinguishability conjecture.
We claim that, informally, the test accuracy of a classifier is close to
the probability that it agrees with an identically-trained classifier on a disjoint train set.
\begin{conjecture}[Agreement Property]
\label{claim:agree}
For certain classifier families $\cF$ and distributions $\cD$,
the test accuracy of a classifier is close to
its \emph{agreement probability} with an independently-trained classifier.
That is,
let $S_1, S_2$ be independent train sets sampled from $\cD^n$,
and let $f_1, f_2$ be classifiers trained on $S_1, S_2$ respectively. Then
\begin{align}
\label{eqn:agree}    
\Pr_{\substack{
S_1 \sim \cD^n\\
f_1 \gets \Train_\cF(S_1)\\
(x, y) \sim \cD
}}[f_1(x) = y]
~\approx~
\Pr_{\substack{
S_1, S_2 \sim \cD^n\\
f_i \gets \Train_\cF(S_i)\\
(x, y) \sim \cD
}}[f_1(x) = f_2(x)]
\end{align}
Moreover, this holds 
with high probability 
over training
$f_1, f_2$:
$\Pr_{\substack{(x, y) \sim \cD
}}[f_1(x) = y]
\approx
\Pr_{\substack{(x, y) \sim \cD
}}[f_1(x) = f_2(x)]
$.
\end{conjecture}

\begin{figure}[ht]
    \centering
    \begin{subfigure}[t]{0.31\textwidth}
        \includegraphics[width=\textwidth]{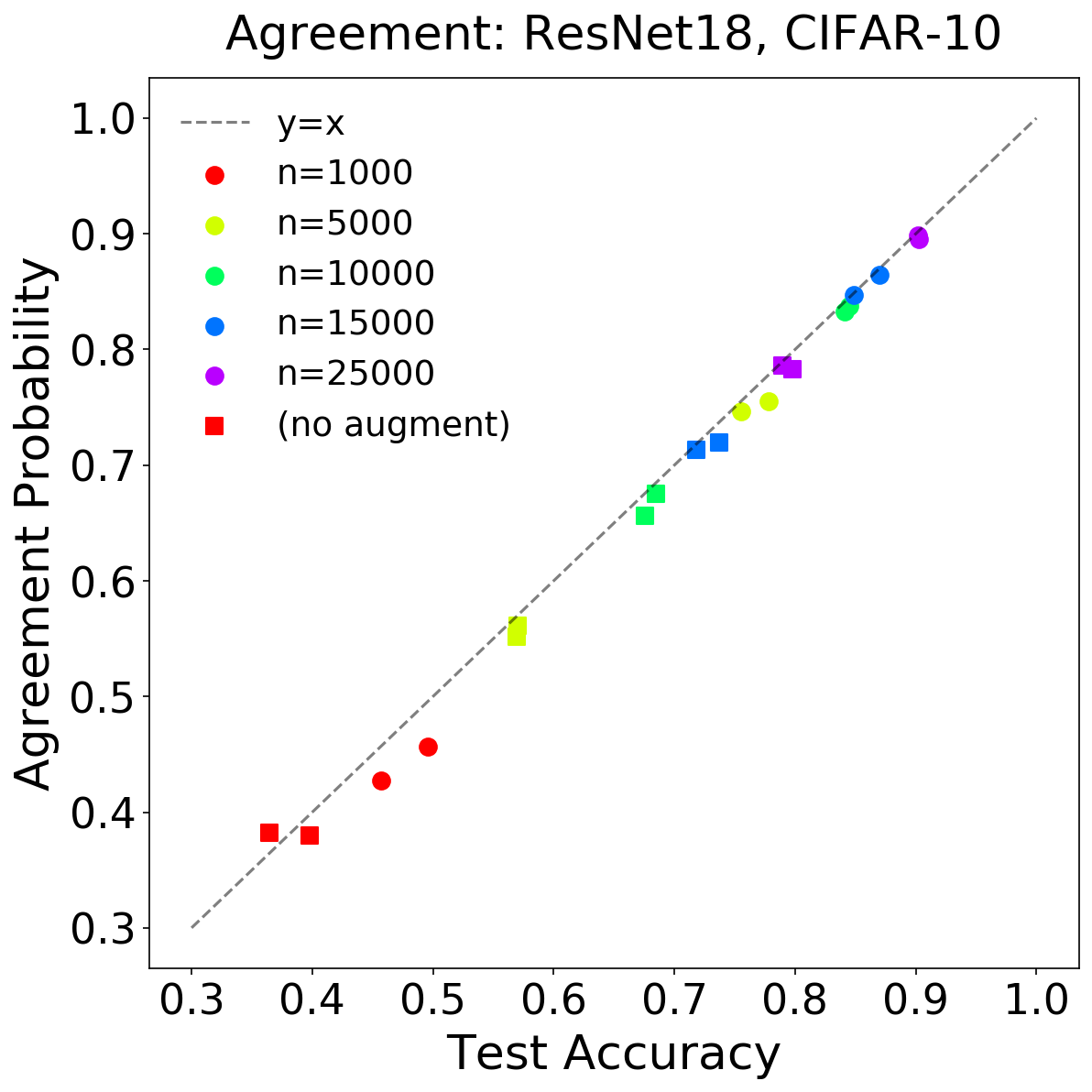}
        \caption{ResNet18 on CIFAR-10.}
        \label{fig:aggr-cf10-cnn}
    \end{subfigure}
    \hfill
    \begin{subfigure}[t]{0.31\textwidth}
        \includegraphics[width=\textwidth]{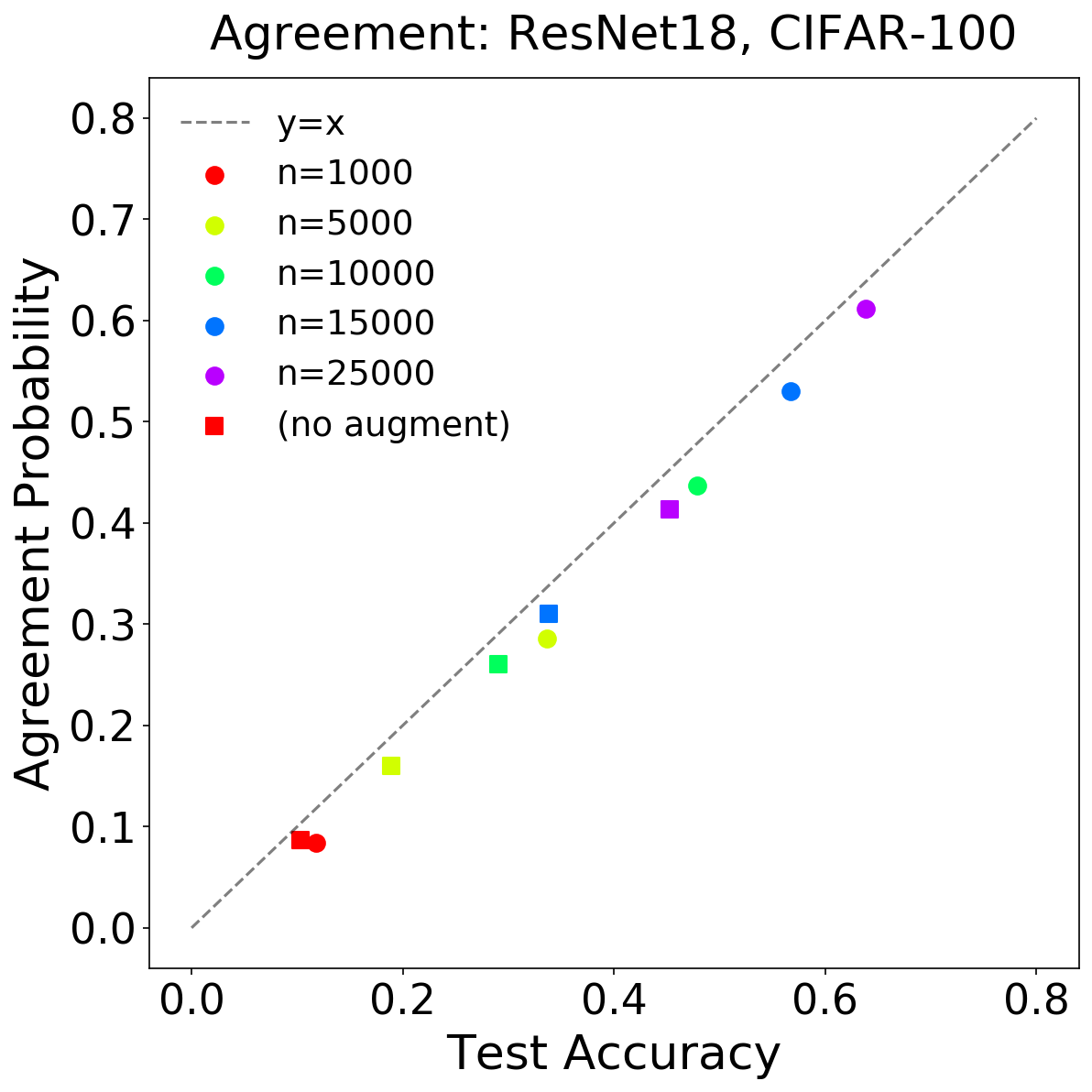}
        \caption{ResNet18 on CIFAR-100.}
        \label{fig:aggr-cf100-cnn}
    \end{subfigure}
    \hfill
    \begin{subfigure}[t]{0.31\textwidth}
        \includegraphics[width=\textwidth]{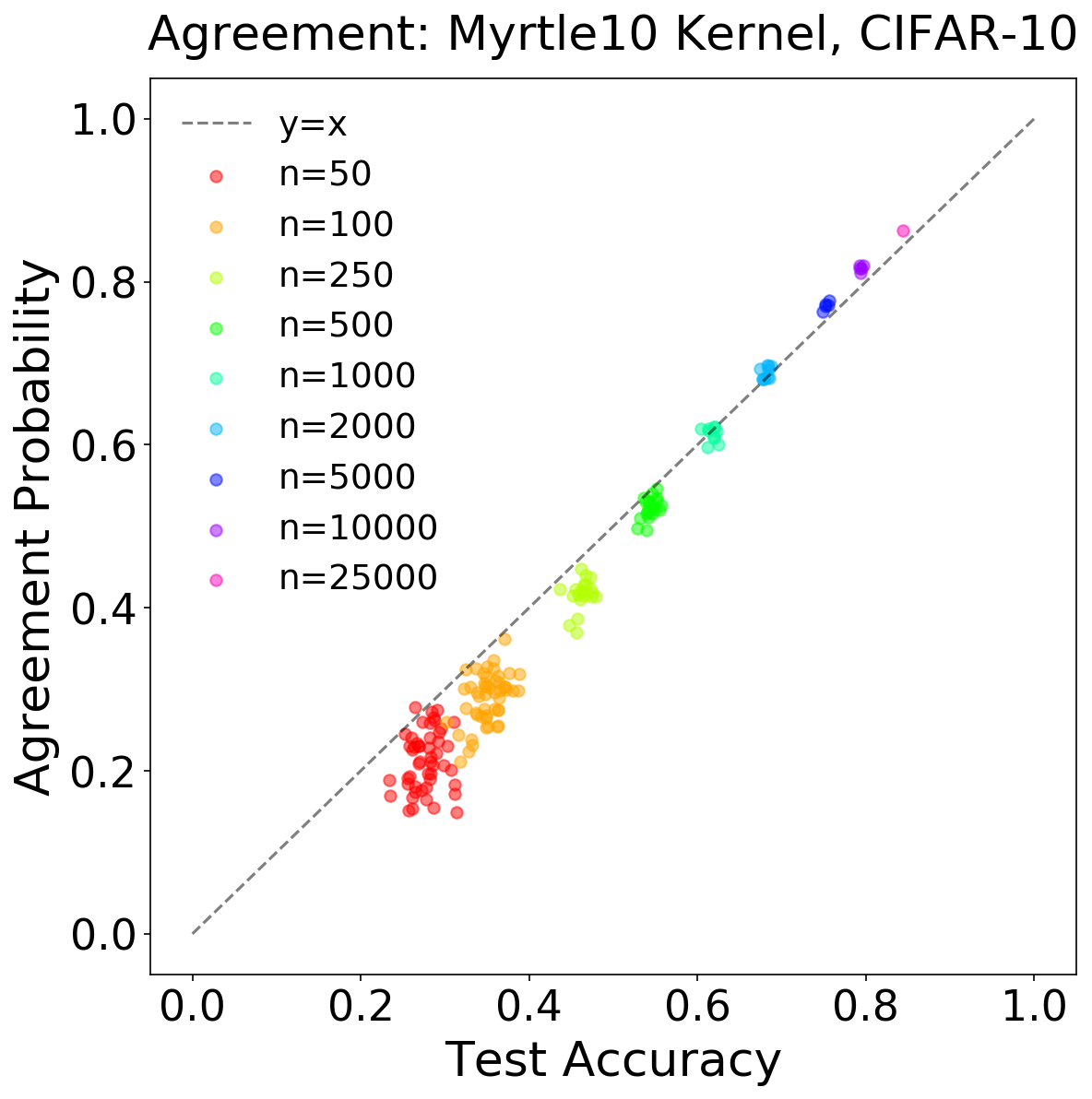}
        \caption{Myrtle Kernel on CIFAR-10.}
        \label{fig:aggr-myrtle}
    \end{subfigure}
    \hfill
    \caption{{\bf Agreement Property on CIFAR-10/100.} 
        For two classifiers trained on disjoint train sets,
        the probability they agree with each other (on the test set)
        is close to their test accuracy.
    }
    \label{fig:aggr-main}
\end{figure}

The agreement property (Conjecture~\ref{claim:agree})
is surprising for several reasons. 
First, suppose we have two classifiers $f_1, f_2$
which were trained on independent train sets,
and both achieve test accuracy say 50\% on a 10-class problem.
That is, they agree with the true label $y(x)$ w.p. 50\%.
Depending on our intuition, we may expect:
(1) They agree with each other much less than they agree with the true label,
since each individual classifier is an independently noisy version of the truth,
or (2) They agree with each other much more than 50\%, since
classifiers tend to have ``correlated'' predictions.
However, neither of these are the case in practice.

Second, it may be surprising that the RHS
of Equation~\ref{eqn:agree} is an estimate of the test error
that requires only unlabeled test examples $x$.
This observation is independently interesting,
and may be relevant for applications in
uncertainty estimation and calibration.
Conjecture~\ref{claim:agree} also provably holds for 1-Nearest-Neighbors
in some settings, under stronger assumptions (Theorem~\ref{thm:aggr} in Appendix~\ref{sec:proofs}).

{\bf Connection to Indistinguishability.}
Conjecture~\ref{claim:agree} is in fact an instantiation of our general indistinguishability conjecture.
Informally, we can ``swap $y$ for $f_2(x)$'' in the LHS of Equation~\ref{eqn:agree}, since they are indistinguishable.
Formally, consider the specific test
\begin{align}
T_{\textrm{agree}}: (x, \hat{y}) \mapsto \1\{f_1(x) = \hat{y}\}
\end{align}
where $f_1 \gets \Train_\cF(\cD^n)$.
The expectation of this test under the Source Distribution
$\cD$ is exactly the LHS of Equation~\ref{eqn:agree},
while the expectation under the Test Distribution $\Dte$ is exactly the RHS.
Thus, Conjecture~\ref{claim:agree} can be equivalently stated as
\begin{align}
\cD \approx^{T_{\textrm{agree}}} \Dte.
\end{align}

\subsection{Experiments}
\begin{figure}[ht]
    \centering
    \begin{subfigure}[t]{0.3\textwidth}
        \includegraphics[width=\textwidth]{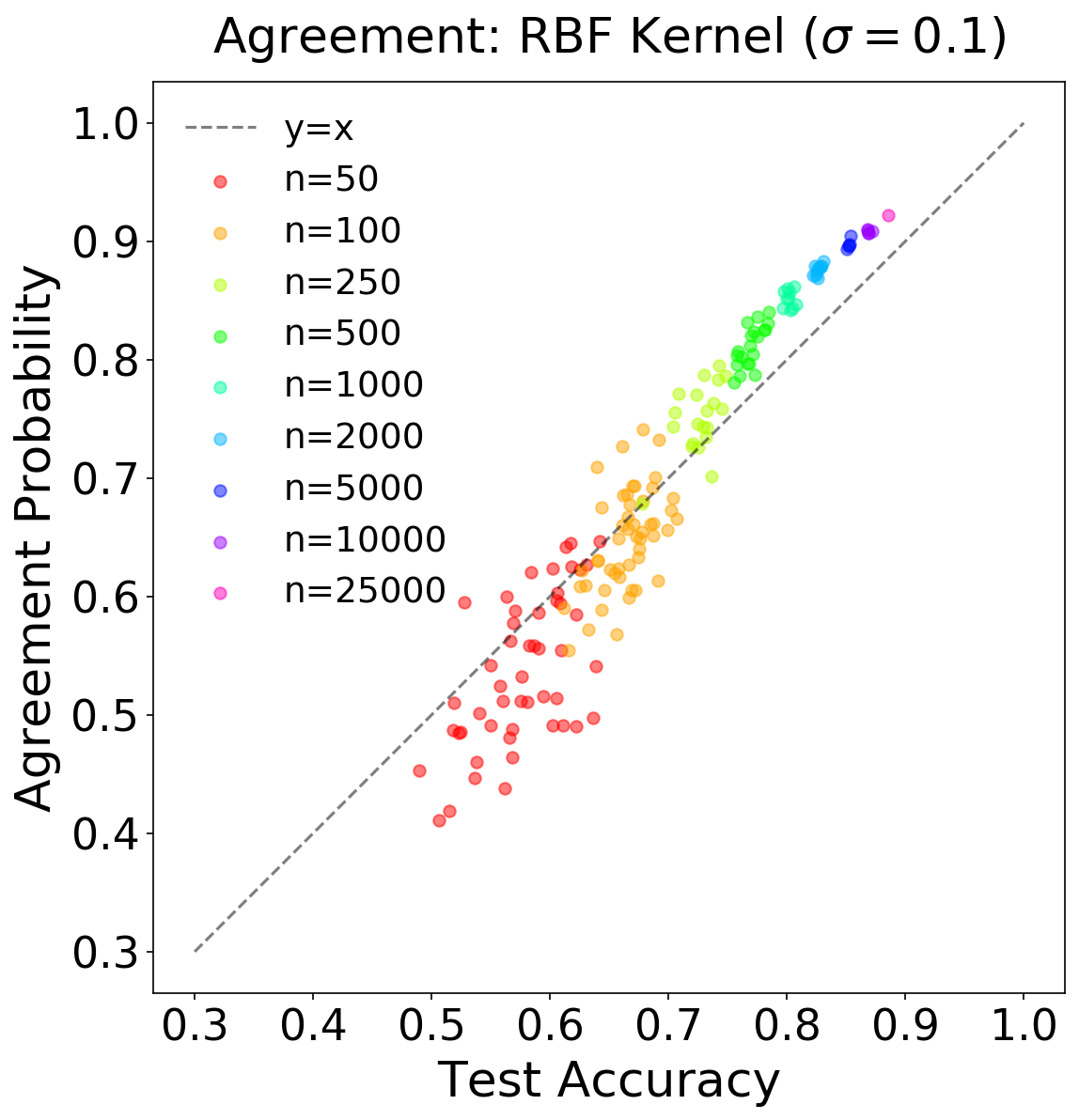}
        \caption{RBF on Fashion-MNIST}
        \label{fig:aggr-fmnist-rbf}
    \end{subfigure}
    \hfill
    \begin{subfigure}[t]{0.3\textwidth}
        \includegraphics[width=\textwidth]{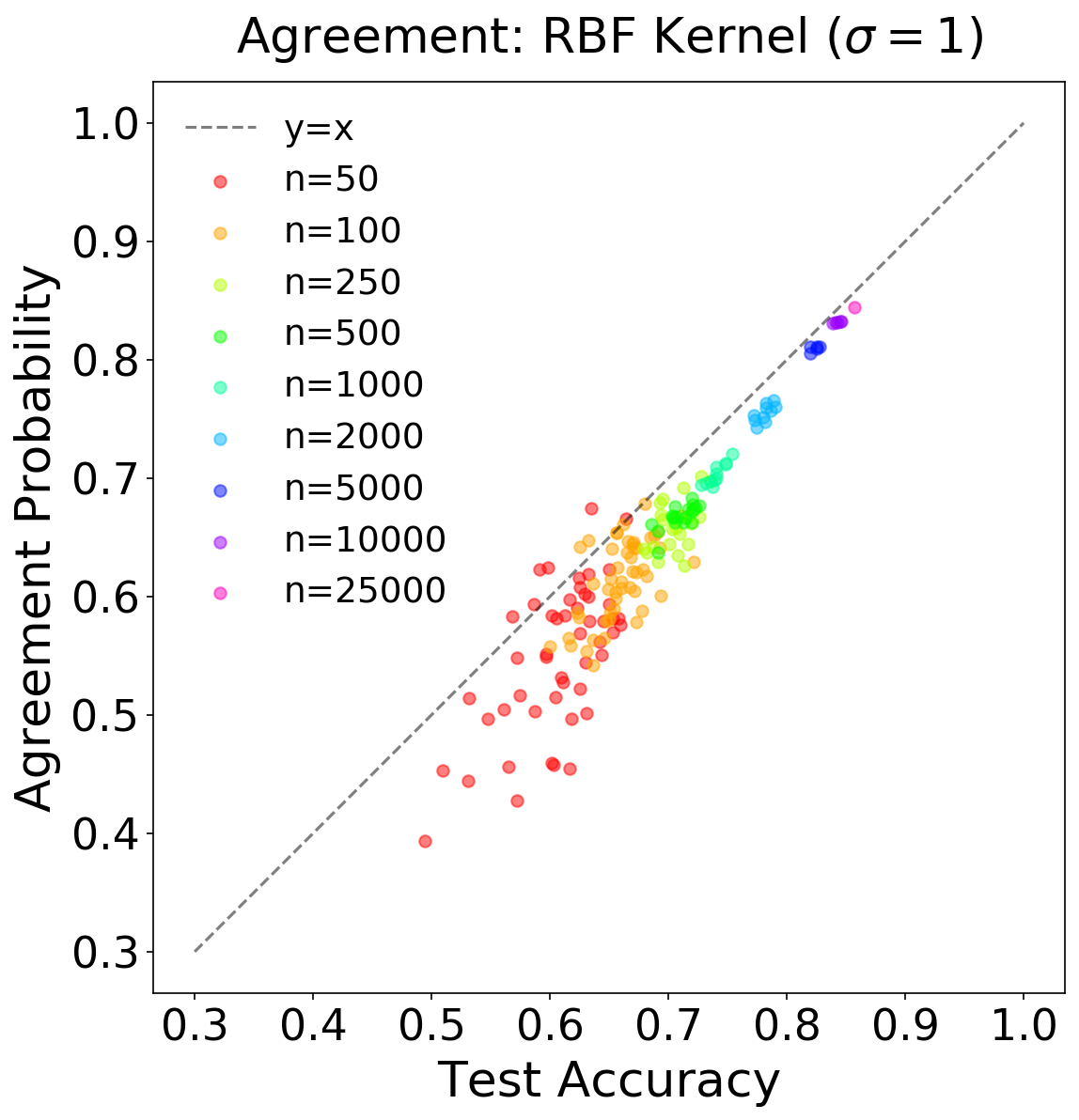}
        \caption{RBF on Fashion-MNIST}
    \end{subfigure}
    \hfill
    \begin{subfigure}[t]{0.3\textwidth}
        \includegraphics[width=\textwidth]{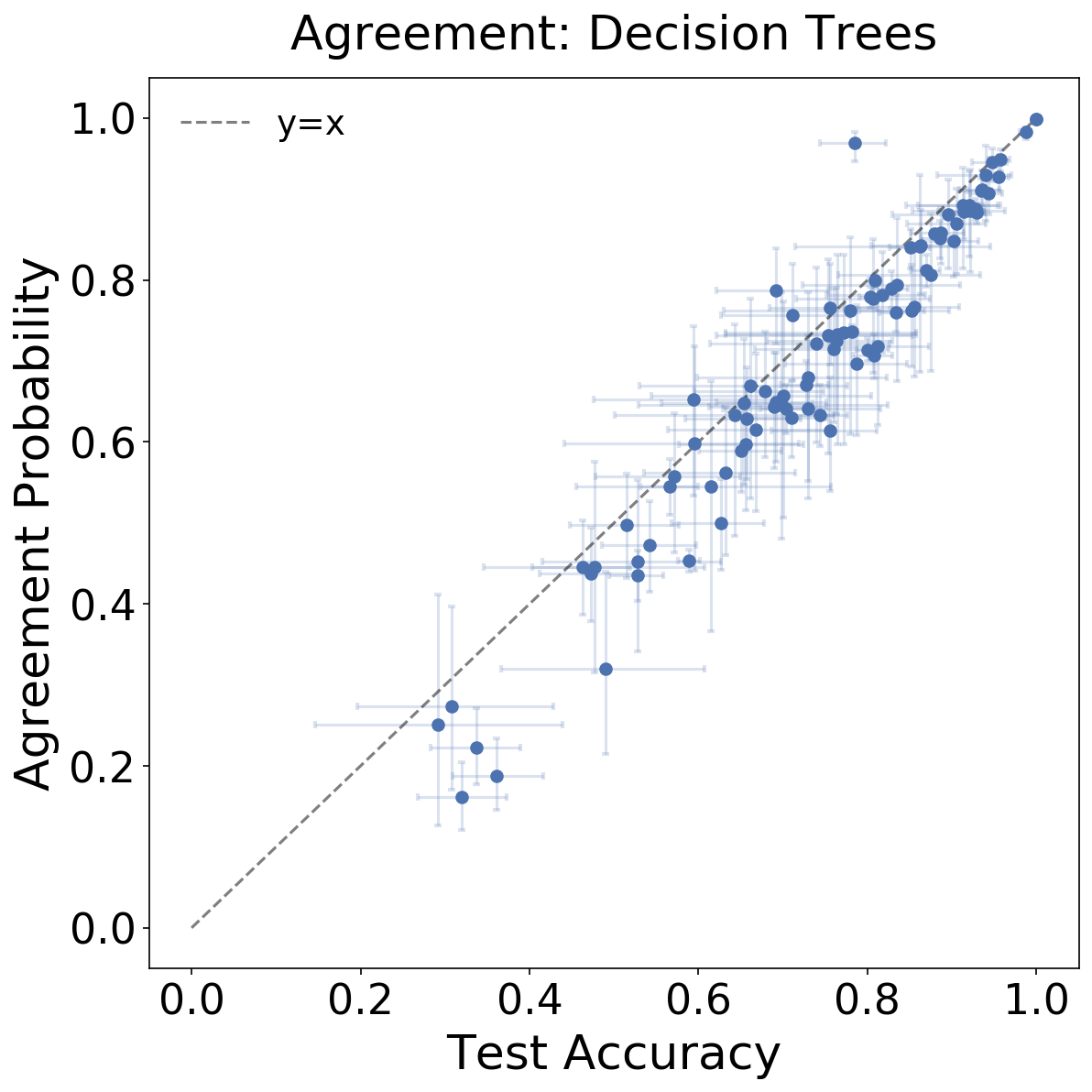}
        \caption{Decision Trees on UCI}
        \label{fig:uci-aggr}
    \end{subfigure}
    \caption{{\bf Agreement Property for RBF and decision trees.} 
        For two classifiers trained on disjoint train sets,
        the probability they agree with each other (on the test set)
        is close to their test accuracy.
        For UCI, each point corresponds to one UCI task,
        and error bars show 95\% Clopper-Pearson confidence intervals in estimating
        population quantities.
    } 
    \label{fig:aggr-main2}
\end{figure}

In our experiments,
we train a pair of classifiers $f_1, f_2$
on random disjoint subsets
of the train set for a given distribution.
Both classifiers are otherwise trained identically,
using the same architecture, number of train-samples $n$, and optimizer.
We then plot the test error of $f_1$ against
the agreement probability $\Pr_{x \sim \TestSet}[f_1(x) = f_2(x)]$.
Figure~\ref{fig:aggr-main} shows experiments with ResNet18 on CIFAR-10
and CIFAR-100,
as well as the Myrtle10 kernel from \citet{shankar2020neural},
with varying number of train samples $n$.
These classifiers are trained with standard-practice training procedures (SGD with standard data-augmentation for ResNets),
with no additional hyperparameter tuning.
Figure~\ref{fig:aggr-main2} shows experiments with the RBF Kernel on Fashion-MNIST,
and decision trees on 92 UCI classification tasks.
The Agreement Property approximately holds for all pairs of identical classifiers,
and continues to hold
even for ``weak'' classifiers
(e.g. when $f_1, f_2$ have high test error).
Full experimental details and further experiments
are in Appendix~\ref{app:agree}.

\subsection{Potential Mechanisms}
We now consider, and refute, several potential mechanisms
which could explain the experimental results of Conjecture~\ref{claim:agree}.

\subsubsection{Bimodal Samples}
\label{sec:bimodal}
A simple model which would exhibit the Agreement Property
is the following:
Suppose test samples $x$ come in two types: ``easy'' or ``hard.''
All classifiers get ``easy'' samples correct,
but they output a uniformly random class on ``hard'' samples.
That is, for a fixed $x$, consider the probability that a
freshly-trained classifier gets $x$ correct.
``Easy'' samples are such that 
\[
\text{For $x \in$ EASY:} \quad
\Pr_{f \gets \Train(\cD^n)}[f(x) = y(x)] = 1
\]
while ``hard'' samples have a uniform distribution on
output classes $[K]$:
\[
\text{For $x \in$ HARD:} \quad
\Pr_{f \gets \Train(\cD^n)}[f(x) = i] =
\frac{1}{K}
~~\forall i \in [K]
\]

Notice that for HARD samples $x$, a classifier $f_1$
agrees with the true label $y$ with exactly the same
probability that it agrees with an independent classifier $f_2$
(because both $f_1, f_2$ are uniformly random on $x$).
Thus, the agreement property (Conjecture~\ref{claim:agree}) holds exactly under this model.
However, this strict decomposition of samples into ``easy'' and ``hard''
does not appear to be the case in the experiments (see Appendix~\ref{app:alt-mech}, Figure~\ref{fig:ptwise-corr}).

\subsection{Pointwise Agreement}
\label{sec:ptwiseagree}
We could more generally posit
that Conjecture~\ref{claim:agree} is true
because the Agreement Property holds
\emph{pointwise} for most test samples $x$.
That is, Equation~\eqref{eqn:agree}
would be implied by:
\begin{equation}
\text{w.h.p. for } (x, y) \sim \cD:
\quad
\Pr_{\substack{
f_1 \gets \Train(\cD^n)\\
}}[f_1(x) = y]
\approx
\Pr_{\substack{
f_1 \gets \Train(\cD^n)\\
f_2 \gets \Train(\cD^n)\\
}}[f_1(x) = f_2(x)]
\label{eqn:ptwise}
\end{equation}
This was the case for the EASY/HARD decomposition above,
but could be true in more general settings.
Equation~\eqref{eqn:ptwise} is a ``pointwise calibration'' property
that would allow estimating the probability of making an error on a test point $x$
by simply estimating the probability that two independent classifiers agree on $x$.
However, we find (perhaps surprisingly) that this is not the case.
That is, Equation~\eqref{eqn:agree} holds on average over $(x, y) \sim \cD$,
but not pointwise for each sample.
We give experiments demonstrating this in Appendix~\ref{app:alt-mech}.
Interestingly, 1-nearest neighbors can
satisfy the agreement property of Claim~\ref{claim:agree} without
satisfying the ``pointwise agreement'' of Equation~\ref{eqn:ptwise}.
It remains an open problem to understand the mechanisms behind Agreement Matching.

\section{Limitations and Ensembles}
\label{sec:limits}

The conjectures presented in this work are not fully specified,
since they do not exactly specify which classifiers or distributions for which they hold.
We experimentally demonstrate instances of these conjectures in various ``natural'' settings in machine learning,
but we do not yet understand which assumptions on the distribution or classifier are required.
Some experiments also deviate slightly from the predicted behavior
(e.g. the kernel experiments in Figures \ref{fig:constant-L} and \ref{fig:aggr-main2}).
Nevertheless, we believe our conjectures capture the essential aspects of the observed behaviors,
at least to first order.
It is an important open question to refine these conjectures and better understand their applications and limitations---
both theoretically and experimentally.

\subsection{Ensembles}
We could ask if all high-performing interpolating methods used in practice satisfy our conjectures.
However, an important family of classifiers which fail our Feature Calibration Conjecture are ensemble methods:
\begin{enumerate}
    \item Deep ensembles of interpolating neural networks \citep{lakshminarayanan2017simple}.
    \item Random forests (i.e. ensembles of interpolating decision trees) \citep{breiman2001random}.
    \item k-nearest neighbors (roughly ``ensembles'' of 1-Nearest-Neighbors) \citep{fix1951discriminatory}.
\end{enumerate}
The pointwise density estimation discussion in Section~\ref{sec:ptwise} sheds some light on these cases.
Notice that these are settings where the ``base'' classifier in the ensemble
obeys Feature Calibration, and in particular, acts as an approximate conditional density 
estimator of $p(y | x)$, as in Section~\ref{sec:ptwise}.
That is, if individual base classifiers $f_i$ approximately act as samples from
$$f_i(x) \sim p( y | x)$$
then for sufficiently many classifiers $\{f_1, \dots, f_k\}$ trained on independent train sets,
the ensembled classifier will act as
$$\textrm{plurality}(f_1, f_2, \dots, f_k)(x) \approx \argmax_y p(y | x)$$
Thus, we believe ensembles fail our conjectures because, in taking the plurality
vote of base classifiers, they are approximating $\argmax_y p(y | x)$ instead of the conditional density $p(y | x)$ itself.
Indeed, in the above examples, we observed that ensemble methods
behave much closer to the Bayes-optimal classifier than their underlying base classifiers
(especially in settings with label noise).

\section{Distributional Generalization: Beyond Interpolating Methods}
\label{sec:gg}

\begin{figure}[p]
\centering
\includegraphics[trim=2cm 2cm 2cm 2cm,clip=true,width=\linewidth]{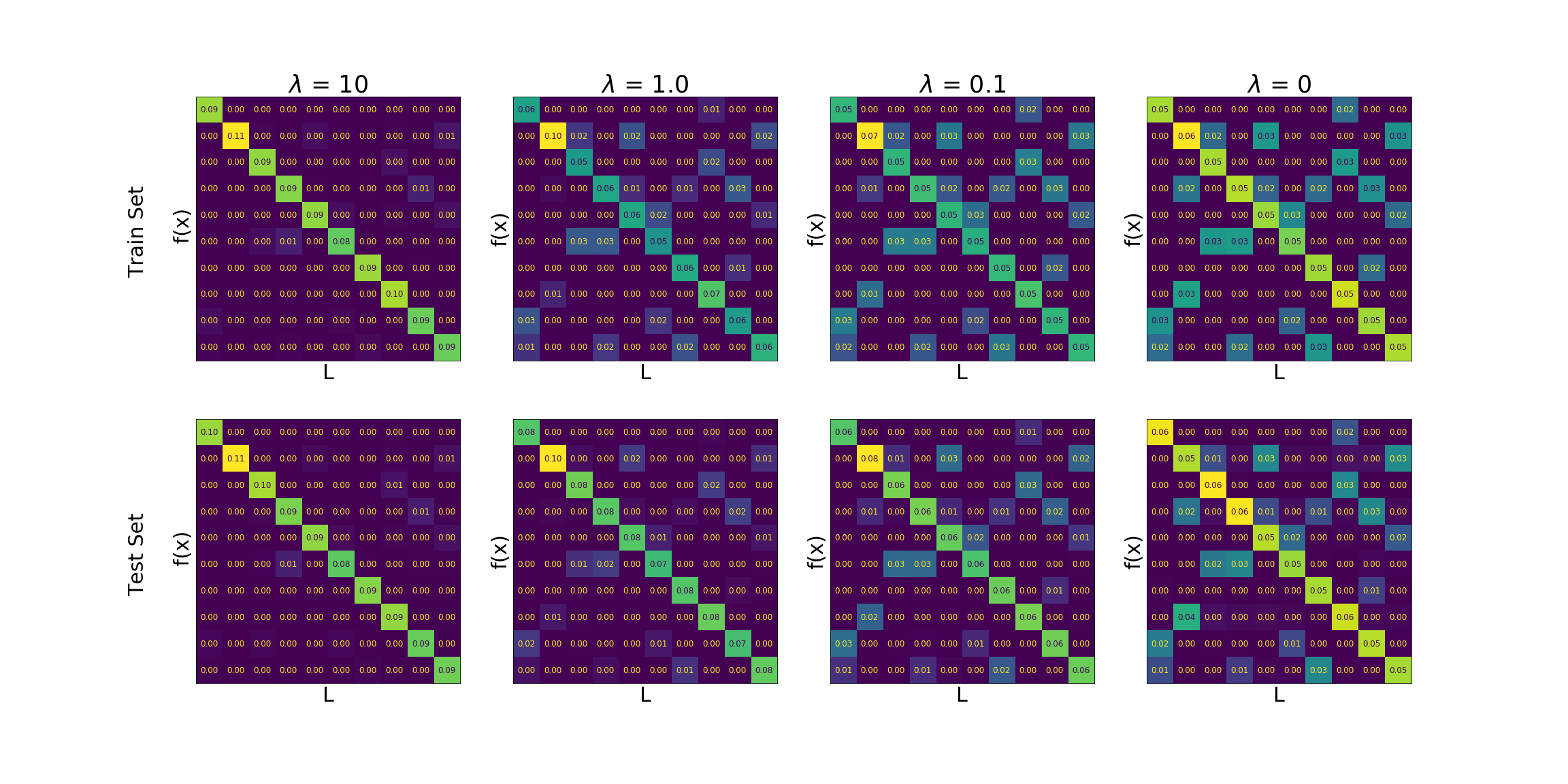}
\caption{{\bf Distributional Generalization for Gaussian Kernel on MNIST.}
We apply label noise from a random sparse confusion to the MNIST train set.
We then train a Gaussian Kernel for classification, with varying $L_2$ regularization $\lambda$.
The top row shows the confusion matrix of predictions $f(x)$ vs true labels $L(x)$ on the train set,
and the bottom row shows the corresponding confusion matrix on the test set.
Larger values of regularization prevents the classifier from fitting label noise on the train set,
and this behavior is mirrored almost identically on the test set.
Note that all classifiers above are trained on the same train set, with the same label noise.
}
\label{fig:kernelgg}
\end{figure}

\begin{figure}[p]
\centering
\includegraphics[trim=2cm 2cm 2cm 2cm,clip=true,width=\linewidth]{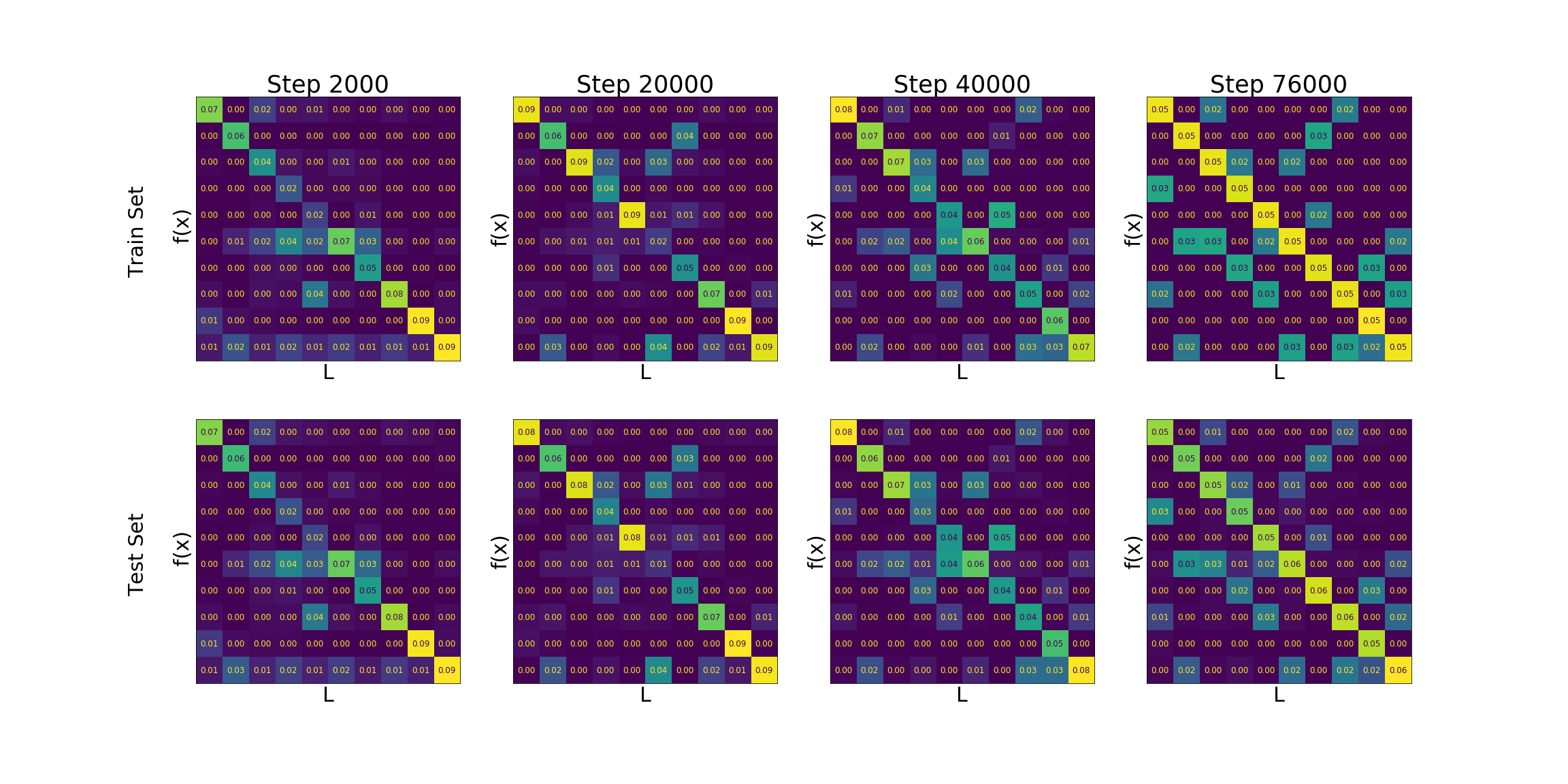}
\caption{{\bf Distributional Generalization for WideResNet on CIFAR-10.}
We apply label noise from a random sparse confusion to the CIFAR-10 train set.
We then train a single WideResNet28-10, and measure its predictions on the train and test sets
over increasing train time (SGD steps).
The top row shows the confusion matrix of predictions $f(x)$ vs true labels $L(x)$ on the train set,
and the bottom row shows the corresponding confusion matrix on the test set.
As the network is trained for longer, it fits more of the noise on the train set,
and this behavior is mirrored almost identically on the test set.
}
\label{fig:cifargg}
\end{figure}

The previous sections have focused primarily on
\emph{interpolating} classifiers, which fit their train sets exactly.
Here we discuss the behavior of non-interpolating methods,
such as early-stopped neural networks and regularized kernel machines, which do not reach 0 train error.

For non-interpolating classifiers, their outputs on the train set $(x, f(x))_{x \sim \TrainSet}$
will \emph{not} match the original distribution $(x, y) \sim \cD$.
Thus, there is little hope that their outputs on the test set will match the original distribution,
and we do not expect the Indistinguishability Conjecture to hold.
However, the Distributional Generalization framework does not require interpolation,
and we could still expect that the train and test distributions are close ($\Dtr \approx^{\cT} \Dte$)
for some family of tests $\cT$.
For example, the following is a possible generalization of Feature Calibration (Conjecture~\ref{conj:approx}).
\begin{conjecture}[Generalized Feature Calibration, informal]
\label{conj:approx-gg}
For trained classifiers $f$,
the following distributions are statistically close for
many partitions $L$ of the domain:
\begin{align}
\underset{
x_i \sim \mathrm{TrainSet}
}{
(L(x_i), f(x_i))
}
\quad\approx\quad
\underset{
x \sim \mathrm{TestSet}
}{
(L(x), f(x))
}
\end{align}

\end{conjecture}
We leave unspecified the exact set of partitions $L$ for which this holds---
unlike Conjecture~\ref{conj:approx},
where we specified $L$ as the set of all distinguishable features. 
In this generalized case, we do not yet understand the appropriate notion of ``distinguishable feature''
\footnote{For example, when considering early-stopped neural networks, it is unclear if the partition $L$
should be distinguishable with respect to the early-stopped network or its fully-trained counterpart.}.
However, we give experimental evidence that suggests some refinement of Conjecture~\ref{conj:approx-gg} is true.

In Figure~\ref{fig:kernelgg} we 
train Gaussian kernel regression on MNIST, with label noise determined by a random sparse confusion matrix
on the train set (analogous to the setting of Figure~\ref{fig:wrn_random};
experimental details in Appendix~\ref{app:experiment}).
We vary the amount of $\ell_2$ regularization,
and plot the confusion matrix of predictions on the train and test sets.
With $\lambda=0$ regularization, the kernel interpolates the noise in the train set exactly,
and reproduces this noise on the test set as expected.
With higher regularization, the kernel no longer interpolates the train set,
but the test and train confusion matrices remain close.
That is, regularization prevents the kernel from fitting the noise on both the train and test sets
in a similar way.
Remarkably, higher regularization yields a classifier closer to Bayes-optimal.
Figure~\ref{fig:cifargg} shows an analogous experiment for neural networks on CIFAR-10,
with early-stopping in place of regularization:
early in training, neural networks do not fit their train set,
but their test and train confusion matrices remain close throughout training.
These experiments suggest that Distributional Generalization is a meaningful
notion even for non-interpolating classifiers.
Formalizing and investigating this further is an 
interesting area for future study.

\section{Conclusion and Discussion}
In this work, we presented a new set of empirical behaviors of standard interpolating classifiers.
We unified these under the framework of Distributional Generalization,
which states that outputs of trained classifiers
on the test set are ``close'' in distribution to their outputs on the train set.
For interpolating classifiers, we stated several formal conjectures
(Conjectures~\ref{conj:approx} and \ref{claim:agree})
to characterize the form of distributional closeness that can be expected.

{\bf Beyond Test Error.}
Our work proposes studying the \emph{entire distribution} of classifier outputs on test samples,
beyond just its test error.
We show that this distribution is often highly structured, and we take steps
towards characterizing it.
Surprisingly, modern interpolating classifiers appear to satisfy
certain forms of distributional generalization ``automatically,''
despite being trained to simply minimize train error.
This even holds in cases when satisfying distributional generalization
is in conflict with satisfying classical generalization--- that is, when
a distributionally-generalizing classifier must necessarily have high test error
(e.g. Experiment 1).
We thus hope that studying distributional generalization will be useful
to better understand modern classifiers, and to understand generalization more broadly.

{\bf Classical Generalization.}
Our framework of Distributional Generalization
can be insightful even to study classical generalization.
That is, even if we ultimately want to understand test error,
it may be easier to do so through distributional generalization. 
This is especially relevant for understanding
the success of interpolating methods,
which pose challenges to classical theories of generalization.
Our work shows new empirical behaviors of interpolating classifiers,
as well as conjectures characterizing these behaviors. This sheds new light on these poorly understood methods,
and could pave the way to better understanding their generalization.

{\bf Interpolating vs. Non-interpolating Methods.}
Our work also suggests that interpolating classifiers should be viewed
as conceptually different objects from non-interpolating ones, even if both have the same test error.
In particular, an interpolating classifier will match certain aspects of the original distribution,
which a non-interpolating classifier will not.
This also suggests, informally, that interpolating methods
should not be seen as methods which simply ``memorize''
their training data in a naive way (as in a look up table) -- rather this ``memorization'' strongly influences the
classifier's decision boundary (as in 1-Nearest-Neighbors).

\subsection{Open Questions}
Our work raises a number of open questions and connections to other areas.
We briefly collect some of them here.
\begin{enumerate}
    \item As described in the Limitations (Section~\ref{sec:limits}),
we do not precisely understand the set of distributions and interpolating classifiers for which our conjectures hold.
We empirically tested a number of ``realistic'' settings, but it is open to state formal assumptions defining these settings.

\item It is open to theoretically prove versions of Distributional Generalization for models beyond 1-Nearest-Neighbors.
This is most interesting in cases where Distributional Generalization is at odds with classical generalization (e.g. Figure~\ref{fig:target-varyp}c).

\item It is open to understand the mechanisms behind the Agreement Property (Section~\ref{sec:agree}),
theoretically or empirically.

\item In some of our experiments (e.g. Section~\ref{sec:ptwise}), ensembling over independent random-initializations
had a similar effect to ensembling over independent train sets.
This is related to works on deep ensembles~\citep{lakshminarayanan2017simple,fort2019deep} as well
as random forests for conditional density estimation
\citep{meinshausen2006quantile,pospisil2018rfcde, athey2019generalized}.
Investigating this further is an interesting area of future work.

\item There are a number of works suggesting ``local'' behavior of neural networks,
and these are somewhat consistent with our locality intuitions in this work.
However, it is open to formally understand whether these intuitions are justified in our setting.

\item We give two families of tests $\cT$ for which our Interpolating Indistinguishability
conjecture (Equation~\ref{eqn:metaconj}) empirically holds. This may not be exhaustive -- there
may be other ways in which the source distribution $\cD$ and test distribution $\Dte$ are close.
Indeed, we give preliminary experiments for another family of tests, based on student-teacher training, in Appendix~\ref{sec:student}.
It is open to explore more ways in which Distributional Generalization holds, beyond the tests presented here.
\end{enumerate}

\subsubsection*{Acknowledgements}
We especially thank Jacob Steinhardt and Boaz Barak for useful discussions during this work.
We thank Vaishaal Shankar for providing the Myrtle10 kernel,
the ImageNet classifiers, and advice regarding UCI tasks.
We thank Guy Gur-Ari for noting the connection to existing work on networks picking up
fine-structural aspects of distributions.
We also thank a number of people for reviewing early drafts
or providing valuable comments, including:
Collin Burns,
Mihaela Curmei,
Benjamin L. Edelman,
Sara Fridovich-Keil,
Boriana Gjura,
Wenshuo Guo,
Thibaut Horel,
Meena Jagadeesan,
Dimitris Kalimeris,
Gal Kaplun,
Song Mei,
Aditi Raghunathan,
Ludwig Schmidt,
Ilya Sutskever,
Yaodong Yu,
Kelly W. Zhang,
Ruiqi Zhong.

Work supported in part by the Simons Investigator Awards of Boaz Barak and Madhu Sudan, and NSF Awards under grants CCF 1565264, CCF 1715187 and IIS 1409097.
Computational resources supported in part by a gift from Oracle,
and Microsoft Azure credits (via Harvard Data Science Initiative).
P.N. supported in part by a Google PhD Fellowship. Y.B is partially supported by MIT-IBM Watson AI Lab.

{\bf Technologies.}
This work was built on the following technologies:
NumPy \citep{oliphant2006guide,van2011numpy,Harris2020},
SciPy \citep{2020SciPy-NMeth},
scikit-learn \citep{scikit-learn},
PyTorch \citep{pytorch},
W\&B \citep{wandb},
Matplotlib \citep{Hunter:2007},
pandas \citep{reback2020pandas, mckinney-proc-scipy-2010},
SLURM \citep{yoo2003slurm},
Figma.
Neural networks trained on NVIDIA V100 and 2080 Ti GPUs.

\newpage
\bibliographystyle{plainnat}
\bibliography{refs}

\newpage
\appendix

\section{Author Contributions}
\label{sec:contrib}
PN designed the initial neural network experiments which initiated this study. PN and YB brainstormed the formalization of the experimental observations and PN devised the final version of the definitions, conjectures, and their framing as a version of generalization. YB designed the experiments to stress test the Feature Calibration Conjecture under various settings and conducted the final experiments that appear in the Feature Calibration section. PN discovered and investigated the Agreement Property and did the Student-Teacher section. PN did the kernel and decision tree experiments, literature review, and nearest-neighbor proofs. Both authors wrote the paper.

\newpage
\section{Student-Teacher Indistinguishability}
\label{sec:student}

Here we show another instance of the Indistinguishability Conjecture,
by giving another way in which the distributions $\Dte$ and $\cD$ are close -- specifically,
we claim they are roughly indistinguishable with respect to \emph{training}.
That is, training a student-network on samples from $(x, y) \sim \cD$
yields a similar model as training on pseudo-labeled samples $(x, f(x)) \sim \Dte$,
as long as the student is ``weaker'' than the teacher $f$.

We specifically consider a setup where a teacher network is trained on $n$ samples,
and a student network is trained on $k \ll n$ samples.
In this setting, we claim that teacher-labeled samples are ``as good as'' real samples to the student -- in that the student achieves similar test accuracy whether trained on real or pseudo-labeled samples.
(In practice we find that $k \leq n/2$ is sufficient, though this limit does not appear to be fundamental.)

To see why this may be surprising, consider a
setting where we use a teacher that is only
only say 80\% accurate.
Now, we may expect that training a student on the pseudo-labeled
distribution will always be worse than training on the true labels.
After all, the pseudo-labels are only 80\% correct.
However, we find that when the student is trained on
less than half the number samples than the teacher was trained on, the student
does just as well as if it were trained on true labels.

\paragraph{Experiments.}
We use a ResNet18 for all student and teacher models.
In Figure~\ref{fig:student},
we use a fixed teacher network trained on $n \in \{5000, 10000\}$ samples from CIFAR-10.
For each $k$, we compare the test error of the following two networks:
\begin{enumerate}
    \item A student ResNet18 on trained on $k$ pseudo-labeled samples
    (call this model $G_{n, k}$: a student trained on $k$ pseudo-samples from a teacher trained on $n$ samples).
    \item A ResNet18 trained on $k$ true samples.
\end{enumerate}
When $k \leq n/2$, 
the test error of the student is close to that of a network trained on true samples.
When $k \gg n/2$ however, the student can distinguish whether it is being trained on
real or pseudo-labeled samples.

Figure~\ref{fig:st_grid} shows test errors of $G_{n, k}$ for all 
$n, k \in \{1000, 2000, 5000, 10000, 15000, 25000\}$.
The test error of $G_{n,k}$ appears to depend only on $\min(n, k)$ --
intuitively, the test error is bottlenecked by the minimum power of student and teacher.

\begin{figure}[h]
    \centering
    \begin{subfigure}[t]{0.4\textwidth}
        \includegraphics[width=\textwidth]{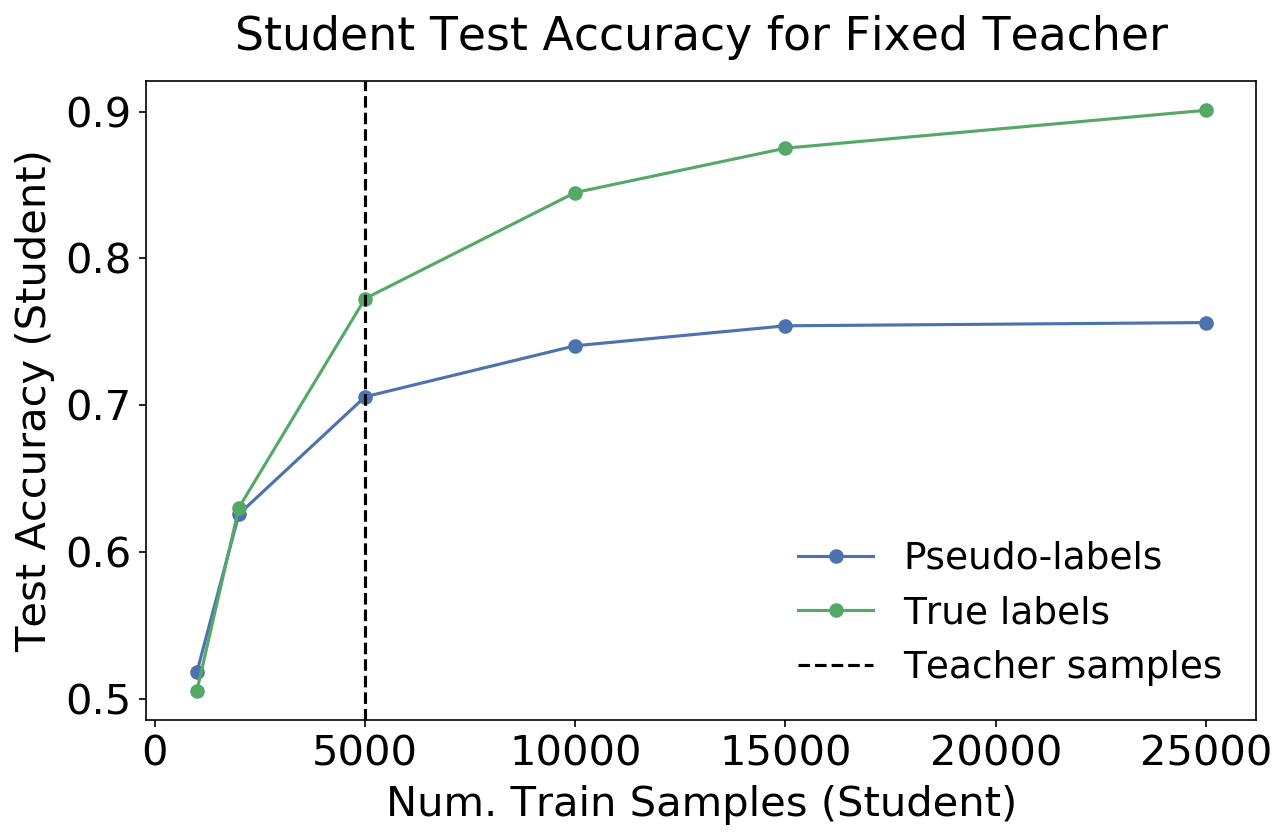}
        \caption{Teacher trained on 5000 samples.}
        \label{fig:fixedt1}
    \end{subfigure}
    \hfill
    \begin{subfigure}[t]{0.4\textwidth}
        \includegraphics[width=\textwidth]{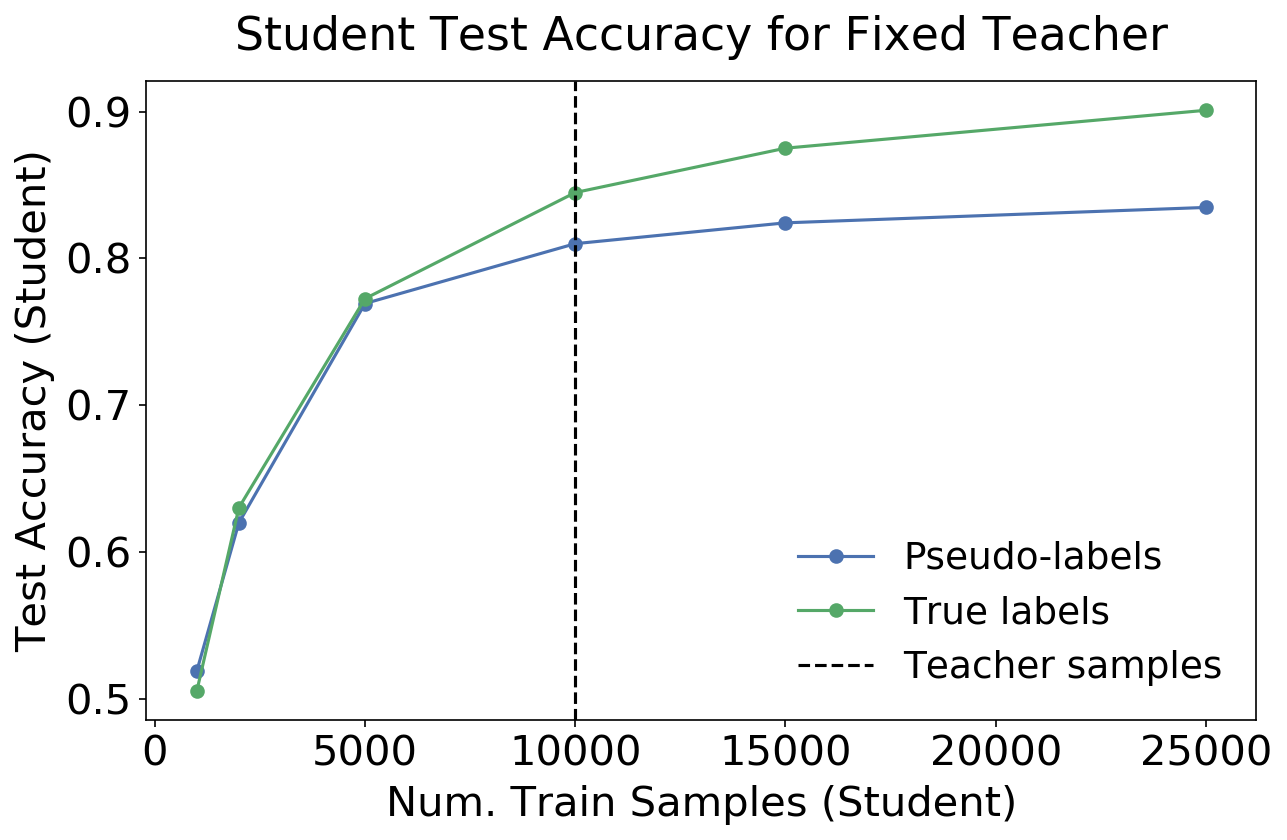}
        \caption{Teacher trained on 10000 samples.}
        \label{fig:todo}
    \end{subfigure}
    \caption{{\bf Pseudo-labeling.} Accuracy of student when trained on true labels vs. pseudo-labels.} 
    \label{fig:student}
\end{figure}

\paragraph{Discussion}
Previous sections considered
tests which were 
asked to distinguish the distributions
$\cD$ and $\Dte$
based on a single sample from either distribution.
Here, we consider a more powerful test, which is given
access to $k$ iid samples from either $\cD$
or $\Dte$.\footnote{
This is not fundamentally different from a single sample test,
via a hybrid argument.
}
This student-teacher indistinguishability is also essentially equivalent to
the following claim: We cannot learn a ResNet-distinguisher between distributions
$\cD$ and $\Dte$, given $k \leq n/2$ samples from each distribution.

\begin{figure}[ht]
    \centering
    \includegraphics[width=0.4\textwidth]{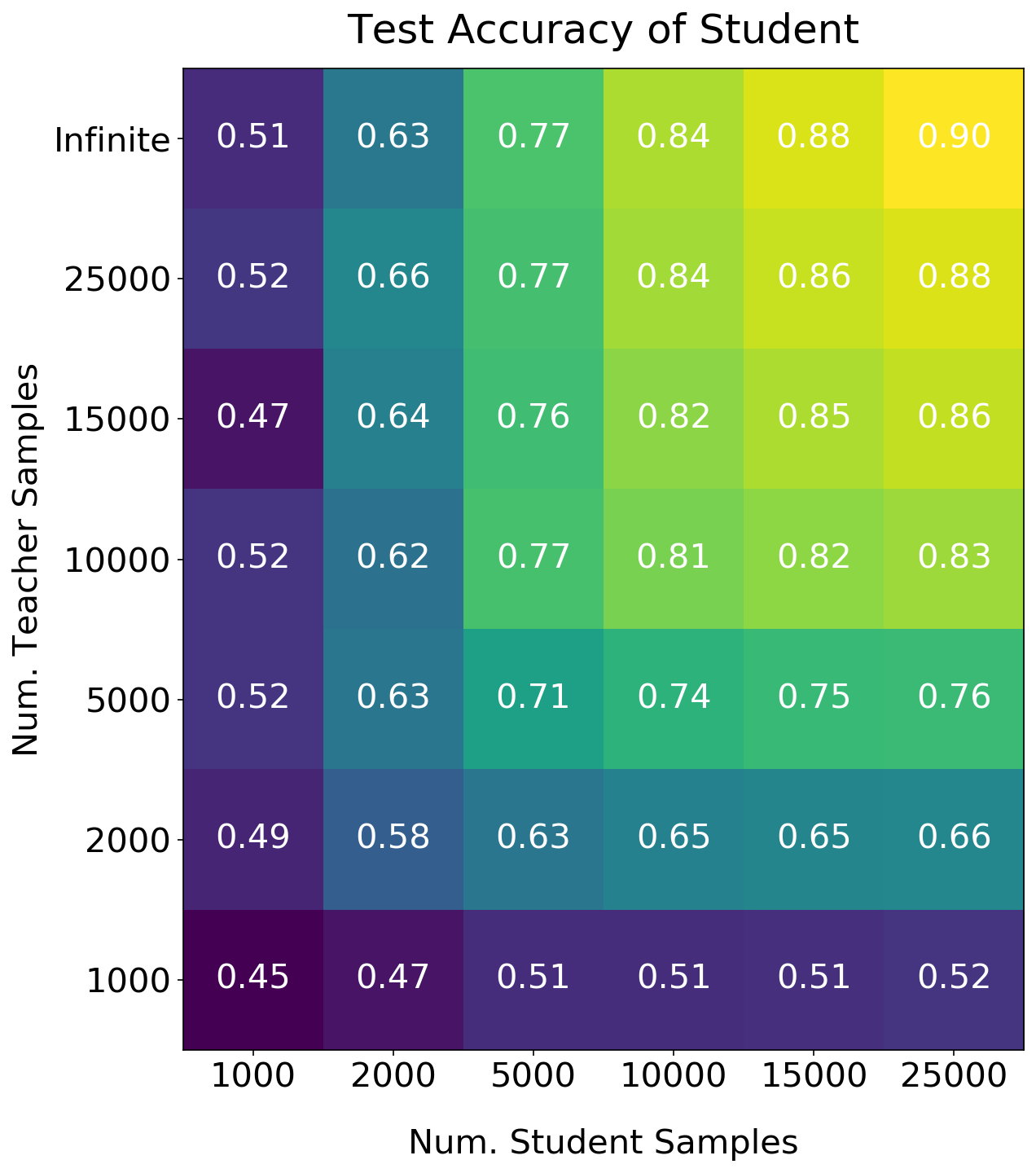}
    \caption{Test error of a student trained on $k$ pseudo-labeled samples (x-axis)
    from a teacher trained on $n$ samples (y-axis). The top row ($n=\infty$) shows
    a student trained on true samples from the distribution.}
    \label{fig:st_grid}
\end{figure}

\newpage
\section{Experimental Details}
\label{app:experiment}

Here we describe general background, and experimental details common to all sections.
Then we provide section-specific details below.

\subsection{Datasets}
We consider the image datasets
CIFAR-10 and CIFAR-100~\citep{krizhevsky2009learning},
MNIST~\citep{lecun1998gradient},
Fashion-MNIST~\citep{xiao2017fashion},
CelebA~\citep{liu2015faceattributes},
and
ImageNet~\citep{ILSVRC15}.

We also consider tabular datasets from the UCI repository~\cite{uci}.
For UCI data, we consider the 121 classification tasks
as standardized in~\citet{fernandez2014we}.
Some of these tasks have very few examples, so we restrict to the 92 classification
tasks from~\citet{fernandez2014we} which have at least $200$ total examples.

\subsection{Models}
\label{app:models}
We consider neural-networks, kernel methods, and decision trees.

\subsubsection{Decision Trees}
We train interpolating decision trees using
a growth rule from Random Forests~\citep{breiman2001random, ho1995random}:
selecting a split based on a random $\sqrt{d}$ 
subset of $d$ features, splitting based on Gini impurity,
and growing trees until all leafs have a single sample.
This is as implemented by Scikit-learn~\cite{scikit-learn}
defaults with
\verb!RandomForestClassifier(n_estimators=1, bootstrap=False)!.

\subsubsection{Kernels}
Throughout this work we consider classification via kernel regression and kernel SVM.
For $M$-class classification via kernel regression, we follow the methodology in e.g. \citet{rahimi2008random,belkin2018understand, shankar2020neural}.
We solve the following convex problem for training:
$$\alpha^* := \argmin_{\alpha \in \R^{N \x M}} ||K\alpha - y||_2^2 + \lambda \alpha^T K \alpha$$
where $K_{ij} = k(x_i, x_j)$ is the kernel matrix of the training points for a kernel function $k$,
$y \in \R^{N \x M}$ is the one-hot encoding of the train labels,
and $\lambda \geq 0$ is the regularization parameter.
The solution can be written
$$\alpha^* = (K + \lambda I)^{-1} y$$
which we solve numerically using SciPy
\verb!linalg.solve! ~\citep{2020SciPy-NMeth}.
We use the explicit form of all kernels involved.
That is, we do not use random-feature approximations~\citep{rahimi2008random}, though we expect they would behave similarly.

The kernel predictions on test points are then given by
\begin{align}
g_\alpha(x) &:= \sum_{i \in [N]} \alpha_i k(x_i, x)\\
f_\alpha(x) &:= \argmax_{j \in [M]} g_\alpha(x)_j
\end{align}
where $g(x) \in \R^{M}$ are the kernel regressor outputs, and $g(x) \in [M]$ is the thresholded classification decision.
This is equivalent to training $M$ separate binary regressors (one for each label), and taking the argmax for classification.
We usually consider \emph{unregularized} regression ($\lambda = 0$), except in Section~\ref{sec:gg}.

For kernel SVM, we use the implementation provided by Scikit-learn~\citep{scikit-learn}
\verb!sklearn.svm.SVC! with a precomputed kernel,
for inverse-regularization parameter $C \geq 0$
(larger $C$ corresponds to smaller regularization).

{\bf Types of Kernels.}
We use the following kernel functions $k: \R^d \x \R^d \to \R_{\geq 0}$.
\begin{itemize}
    \item Gaussian Kernel (RBF): $k(x_i, x_j) = \exp(-\frac{||x_i - x_j||_2^2}{2\widetilde{\sigma}^2})$.
    \item Laplace Kernel: $k(x_i, x_j) = \exp(-\frac{||x_i - x_j||_2}{\widetilde{\sigma}})$.
    \item Myrtle10 Kernel: This is the compositional kernel introduced by \citet{shankar2020neural}.
    We use their exact kernel for CIFAR-10.
\end{itemize}
For the Gaussian and Laplace kernels, we parameterize bandwidth by $\sigma := \widetilde{\sigma}/\sqrt{d}$.
We use the following bandwidths, found by cross-validation to maximize the unregularized test accuracy:
\begin{itemize}
    \item MNIST: $\sigma=0.15$ for RBF kernel.
    \item Fashion-MNIST: $\sigma=0.1$ for RBF kernel. $\sigma = 1.0$ for Laplace kernel.
    \item CIFAR-10: Myrtle10 Kernel from \citet{shankar2020neural}, and $\sigma=0.1$ for RBF kernel.
\end{itemize}

\subsubsection{Neural Networks}
We use 4 different neural networks in our experiments. We use a multi-layer perceptron, and three different Residual networks. 

{\bf MLP:} We use a Multi-layer perceptron or a fully connected network with 3 hidden layers with $512$ neurons in each layer. A hidden layer is followed by a BatchNormalization layer and ReLU activation function.

{\bf WideResNet:} We use the standard WideResNet-28-10 described in \citet{zagoruyko2016wide}. Our code is based on \href{https://github.com/hysts/pytorch_image_classification/blob/master/pytorch_image_classification/models/cifar/wrn.py}{this repository}.

{\bf ResNet50:} We use a standard ResNet-50 from the PyTorch library \citep{paszke2017automatic}.

{\bf ResNet18:} We use a modification of ResNet18 \cite{he2016deep} adapted to CIFAR-10 image sizes. Our code is based on \href{https://github.com/kuangliu/pytorch-cifar/blob/master/models/resnet.py}{this repository}.

For Experiment \ref{exp:intro1} and \ref{exp:intro2} and Section \ref{sec:dist-features}, the hyperparameters used to train the above networks are given in Table~\ref{table:hyperparams}.

\begin{table}
\centering
\begin{tabular}{ |c|c|c|c|c| } 
 \hline
  & {\bf MLP} & {\bf ResNet18} & {\bf WideResNet-28-10}  & {\bf ResNet50} \\ 
 \hline
 {\bf Batchsize} & 128 & 128 & 128 & 32 \\ 
 \hline
 {\bf Epochs} & 820 & 200 & 200 & 50 \\ 
 \hline
 {\bf Optimizer} & \makecell{Adam \\ ($\beta_1 = 0.9, \beta_2 = 0.999$)} & \makecell{SGD + \\ Momentum (0.9)} & \makecell{SGD + \\ Momentum (0.9)} & SGD \\ 
 \hline
 \makecell{{\bf Learning rate} \\ (LR) schedule} & \makecell{Constant LR $=0.001$} & \makecell{Inital LR$=0.05$ \\ scale by 0.1 at \\ epochs $(80, 120)$} & \makecell{Inital LR$=0.1$ \\ scale by 0.2 at \\ epochs $(80, 120, 160)$} & 
 \makecell{Initial LR $=0.001$, \\ scale by $0.1$ \\ if training loss stagnant \\ for $2000$ gradient steps}\\ 
 \hline
 {\bf \makecell{Data \\ Augmentation}} & \multicolumn{4}{c|}{Random flips + RandomCrop(32, padding=4)}\\ 
 \hline
 {\bf \makecell{CIFAR-10 Error}} & $\sim 40\%$ & $\sim 8\%$ & $\sim 4\%$ & N/A \\ 
 \hline
\end{tabular}
\centering
 \caption{Hyperparameters used to train the neural networks and their errors on the unmodified CIFAR-10 dataset}
 \label{table:hyperparams}
\end{table}

\newpage
\section{Feature Calibration: Appendix}
\label{app:density}

\subsection{A guide to reading the plots}
All the experiments in support of Conjecture \ref{conj:approx} (experiments in Section \ref{sec:dist-features} and the Introduction) involve various quantities which we enumaerate here %

\begin{enumerate}
    \item Inputs $x$: Each experiment involves inputs from a standard dataset like CIFAR-10 or MNIST. We use the standard train/test splits for every dataset.
    \item Distinguishable feature $L(x)$: This feature depends only on input $x$. We consider various features like the original classes itself, a superset of classes (as in coarse partition) or some secondary attributes (like the binary attributes provided with CelebA)
    \item Output labels $y$: The output label may be some modification of the original labels. For instance, by adding some type of label noise, or a constructed binary task as in Experiment \ref{exp:intro1}
    \item Classifier family $F$: We consider various types of classifiers like neural networks trained with gradient based methods, kernel and decision trees.
\end{enumerate}

In each experiment, we are interested in two joint densities $(y, L(x))$, which depends on our dataset and task and is common across train and test, and $(f(x), L(x))$ which depends on the interpolating classifiers outputs on the \emph{test} set. Since  $y, L(x)$ and $f(x)$ are discrete, we will look at their discrete joint distributions. We sometimes refer to $(y, L(x))$ as the train joint density, as at interpolation $(y, L(x)) = (f(x), L(x))$ for all training inputs $x$. We also refer to $(f(x), L(x))$ as the test density, as we measure this only on the test set.

\subsection{Experiment \ref{exp:intro1}}
\label{app:intro-exp-1}

{\bf Experimental details:} We now provide further details for Experiment \ref{exp:intro1}. We first construct a dataset from CIFAR-10 that obeys the joint density $(y, L(x))$ shown in Figure \ref{fig:intro} left panel. We then train a WideResNet-28-10 (WRN-28-10) on this modified dataset to zero training error. The network is trained with the hyperparameters described in Table \ref{table:hyperparams}. We then observe the joint density $(f(x), L(x))$ on the test images and find that the two joint densities are close as shown in Figure \ref{fig:intro}.

We now consider a modification of this experiment as follows:

\begin{experiment}
\label{exp:intro2}
Consider the following distribution over images $x$ and binary labels $y$.
Sample $x$ as a uniformly random CIFAR-10 image,
and sample the label as $p(y | x) = \textrm{Bernoulli}(\texttt{CIFAR\_Class(x)}/10)$.
That is, if the CIFAR-10 class of $x$ is $k \in \{0, 1, \dots 9\}$,
then the label is $1$ with probability $(k/10)$ and $0$ otherwise. 
Figure~\ref{fig:bnry-inc-wrn} shows this joint distribution of $(x, y)$.
As before, train a WideResNet to 0 training error on this distribution.
\end{experiment}

In this experiment too, we observe that the train and test joint densities are close as shown in Figure \ref{fig:bnry-inc-wrn}.

\begin{figure}
    \centering
    \includegraphics[width=\linewidth]{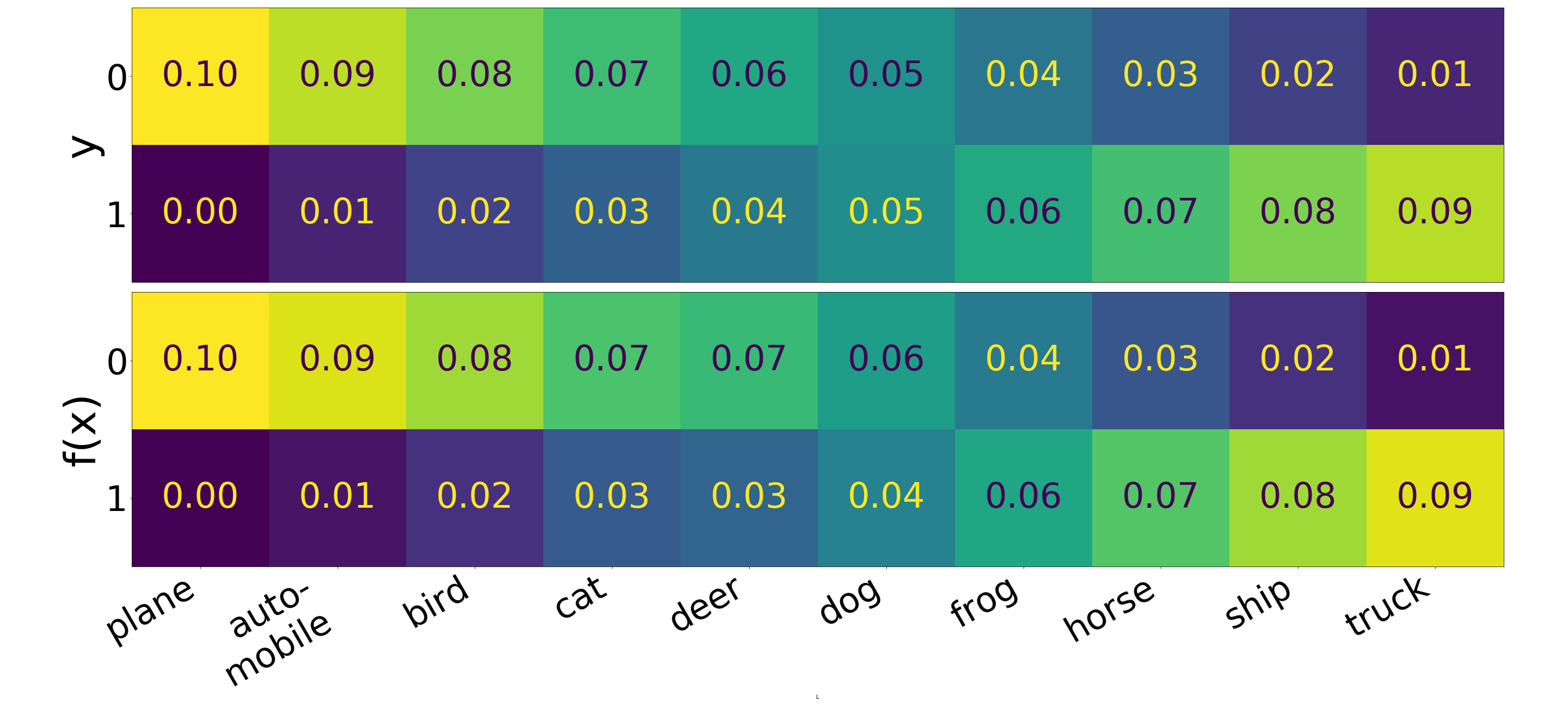}
    \caption{{\bf Distributional Generalization in Experiment~\ref{exp:intro2}.}~~
    Joint densities of the distributions involved in Experiment \ref{exp:intro2}.
    The top panel shows the joint density of labels on the train set: $(\texttt{CIFAR\_Class(x)}, y)$.
    The bottom panels shows the joint density of classifier predictions on the test set: $(\texttt{CIFAR\_Class(x)}, f(x))$.
    Distributional Generalization claims that these two joint densities are close.
    }
    \label{fig:bnry-inc-wrn}
\end{figure}

Now, we repeat the same experiment, but with an MLP instead of WRN-28-10. The training procedure is described in Table~\ref{table:hyperparams}. This MLP has an error on $~ 37\%$ on the original CIFAR-10 dataset.

\begin{figure}[th]
    \includegraphics[width=\textwidth]{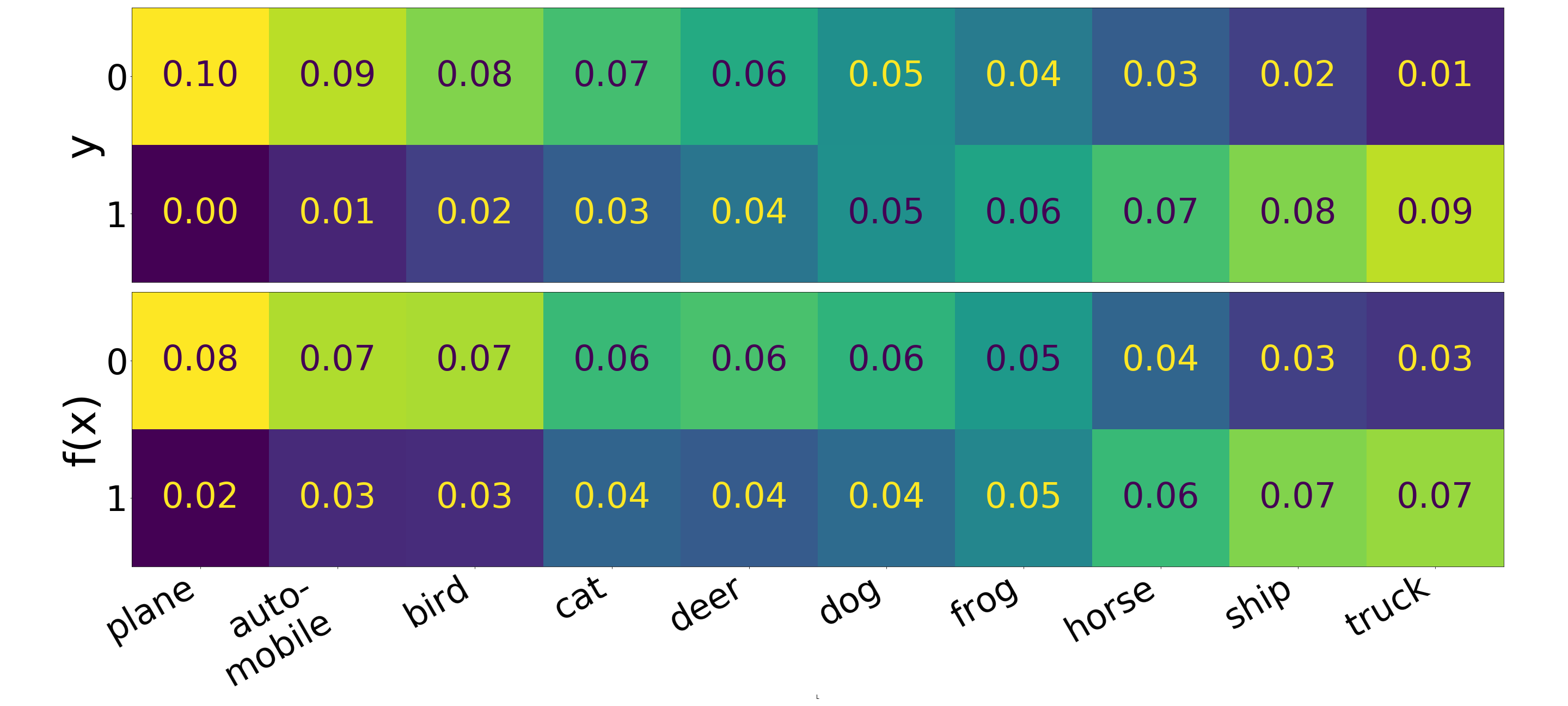}
    \centering
    \caption{Joint density of $(y, \textrm{Class}(x))$, top,
    and $(f(x), \textrm{Class}(x))$, bottom, for test samples $(x, y)$
    from Experiment~\ref{exp:intro2} for an MLP.
    }
    \label{fig-app:bnry-inc-mlp}
\end{figure}

Since this MLP has poor accuracy on the original CIFAR-10 classification task, it does not form a distinguishable partition for it. As a result, the train and test joint densities (Figure \ref{fig-app:bnry-inc-mlp})  do not match as well as they did for WRN-28-10.

\subsection{Constant Partition}
We now describe the experiment for a constant partition $L(x)=0$. For this experiment, we first construct a dataset based on CIFAR-10 that has class-imbalance. For class $k \in \{0...9\}$, sample $(k+1)\times 500$ images from that class. This will give us a dataset where classes will have marginal distribution $p(y = \ell) \propto \ell+1$ for classes $\ell \in [10]$, as shown in Figure \ref{fig:constant-L}. We do this both for the training set and the test set, to keep the distribution $\cD$ fixed.

Now, we train the MLP, ResNet-18 and RBF Kernel on this dataset. We plot the resulting $p(f(x))$ for each of these models below. That is, we plot the fraction of test images for which the network outputs $\{0, 1... 9\}$ respectively. As predicted, the networks closely track the train set.

\subsection{Class Partition}

\subsubsection{Neural Networks and CIFAR-10}

We now provide experiments in support of Conjecture \ref{conj:approx} when the class itself is a distinguishable partition. We know that WRN-28-10 achieves an error of $~4\%$ on this dataset. Hence, the original labels in CIFAR-10 form a distinguishable partition for this dataset. To demonstrate that Conjecture \ref{conj:approx} holds, we consider different structured label noise on the CIFAR-10 dataset. To do so, we apply a variety of confusion matrices to the data. That is, for a confusion matrix $C: 10 \times 10$ matrix, the element $c_{ij}$ gives the joint density that a randomly sampled image had original label $j$, but is flipped to class $i$. For no noise, this would be an identity matrix. 

We begin by a simple confusion matrix where we flip only one class $0 \rightarrow 1$ with varying probability $p$. Figure \ref{fig:target-varyp} shows one such confusion matrix for $p=0.4$. We then train a WideResNet-28-10 to zero train error on this dataset. We use the hyperparameters described in \ref{app:models} We find that the classifier outputs on the test set closely track the confusion matrix that was applied to the distribution. Figure \ref{fig:target-varyp} shows that this is independent of the value of $p$ and continues to hold for $p \in [0, 1]$.

To show that this is not dependent on the particular class used, we also show that the same holds for a random confusion matrix. We generate a sparse confusion matrix as follows. We set the diagonal to $0.5$. Then, for every class $j$, we pick any two random classes for and set them to $0.2$ and $0.3$. We train a WRN-28-10 on it and report the test confusion matrix. The resulting train and test densities are shown in Figure \ref{fig:wrn_random} and also below for clarity.

\subsubsection{Decision Trees}

Similar results hold for decision trees; here we show experiments
on two UCI tasks: \verb!wine! and \verb!mushroom!.

The \verb!wine! task is a 3-way classification problem:
to identify the cultivar of a given wine (out of 3 cultivars), given 13 physical attributes describing the wine.
Figure~\ref{fig-app:dtwine} shows an analogous experiment with label noise taking class $1$
to class $2$.

The \verb!mushroom! task is a 2-way classification problem:
to classify the type of edibility of a mushroom (edible vs poisonous)
given 22 physical attributes (e.g. stalk color, odor, etc).
Figure~\ref{fig-app:dtmushroom} shows an analogous experiment
with label noise flipping class $0$ to class $1$.

\begin{figure}[th]
    \includegraphics[width=\linewidth]{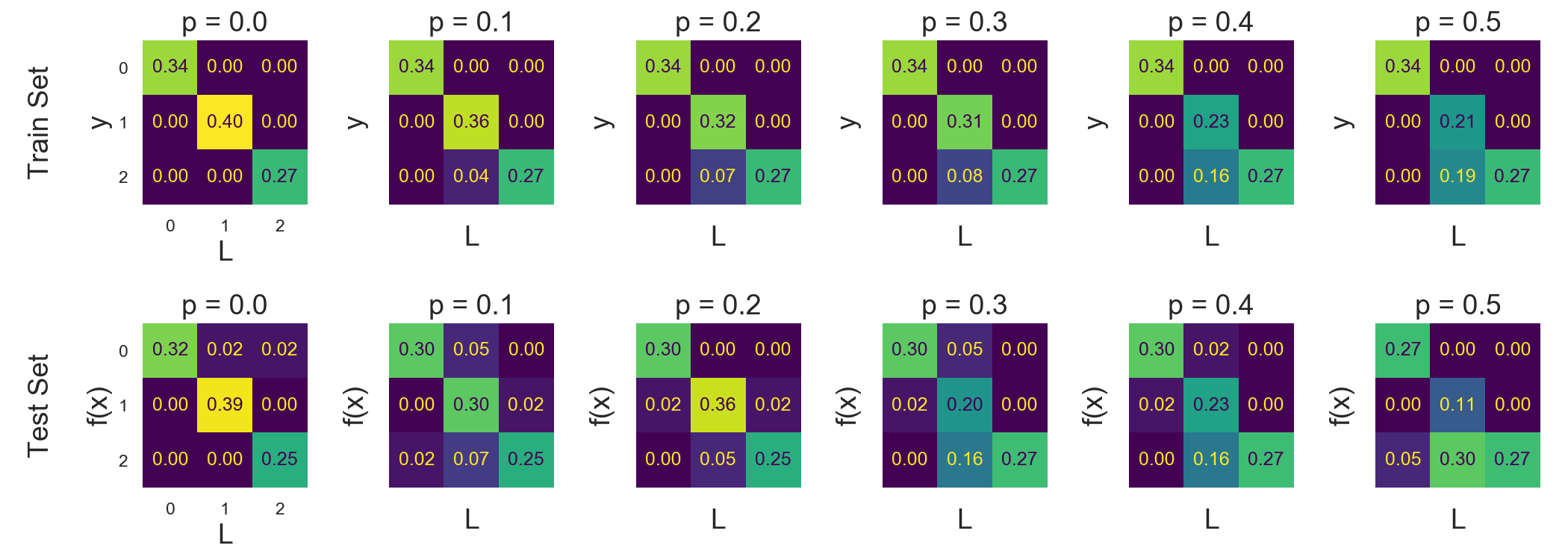}
    \centering
    \caption{Decision trees on UCI (wine).
    We add label noise that takes class $1$ to class $2$ with probability $p \in [0, 0.5]$.
    Each column shows the test and train confusion matrices for a given $p$.
    Note that this decision trees achieve high accuracy on this task with no label noise (leftmost column).
    We plot the empirical joint density of the train set, and not the population joint density of the train distribution,
    and thus the top row exhibits some statistical error due to small-sample effects.
    }
    \label{fig-app:dtwine}
\end{figure}

\begin{figure}[th]
    \includegraphics[width=\linewidth]{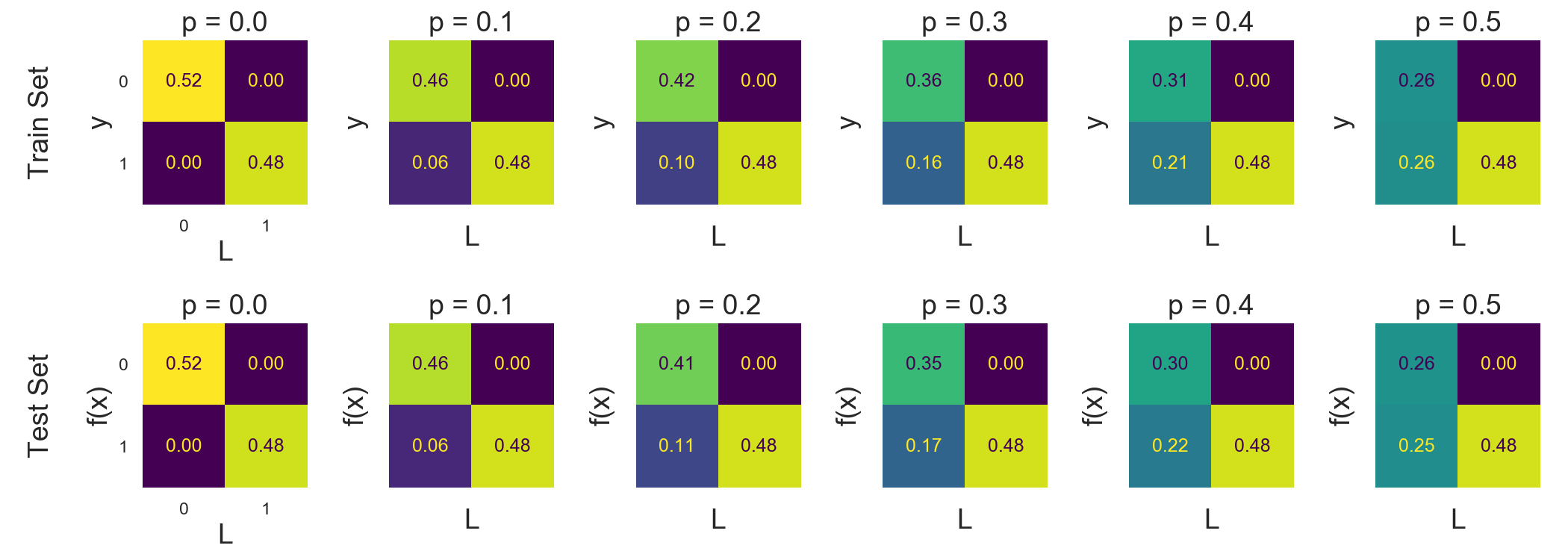}
    \centering
    \caption{Decision trees on UCI (mushroom).
    We add label noise that takes class $1$ to class $2$ with probability $p \in [0, 0.5]$.
    Each column shows the test and train confusion matrices for a given $p$.
    Note that this decision trees achieve high accuracy on this task with no label noise (leftmost column).
    }
    \label{fig-app:dtmushroom}
\end{figure}

\subsection{Multiple Features}
\label{app:celeba}
For the CelebA experiments, we train a ResNet-50 to predict the attribute \{Attractive, Not Attractive\}. We choose this attribute because a ResNet-50 performs poorly on this task (test error $\sim20\%$) and has good class balance. We choose an attribute with poor generalization because the conjecture would hold trivially for if the network generalizes well. We initialize the network with a pretrained ResNet-50 from the PyTorch library \cite{paszke2017automatic} and use the hyperparameters described in Section \ref{app:models} to train on this attribute. We then check the train/test joint density with various other attributes like Male, Wearing Lipstick etc. Note that the network is not given any label information for these additional attributes, but is calibrated with respect to them. That is, the network says $\sim 30\%$ of images that have `Heavy Makeup' will be classified as `Attractive', even if the network makes mistakes on which particular inputs it chooses to do so. Loosely, this can be viewed as the network performing 1-Nearest-Neighbor classification in a metric space that is well separated for each of these distinguishable features.

\subsection{Coarse Partition}
We now consider cases where the original classes do not form a distinguishable partition for the classifier in consideration. That is, the classifier is not powerful enough to obtain low error on the original dataset, but can perform well on a coarser division of the classes. 

To verify this, we consider a division of the CIFAR-10 classes into Objects \{airplane, automobile, ship, truck\} vs Animals \{cat, deer, dog, frog\}. An MLP trained on this problem has low error ($\sim8\%$), but the same network performs poorly on the full dataset ($\sim37\%$ error). Hence, Object vs Animals forms a distinguishable partition with MLPs. In Figure \ref{fig:oa-cifar-mlp}, we show the results of training an MLP on the original CIFAR-10 classes. We see that the network mostly classifies objects as objects and animals as animals, even when it might mislabel a dog for a cat.

We perform a similar experiment for the RBF kernel on Fashion-MNIST, with partition $\text{\{clothing, shoe, bag\}}$,
in Figure~\ref{fig:fashion}.

{\bf ImageNet experiment.}
In Table~\ref{tab:imagenet} we provide results of the terrier experiment in the body,
for various ImageNet classifiers.
We use publicly available pretrained ImageNet models from 
\href{https://github.com/Cadene/pretrained-models.pytorch}{this repository},
and use their evaluations on the ImageNet test set.

\begin{table}
\centering
\begin{tabular}{lrrrrr}
\toprule
Model &  AlexNet &  ResNet18 &  ResNet50 &  BagNet8 &  BagNet32 \\
\midrule
ImageNet Accuracy                            &    0.565 &     0.698 &     0.761 &    0.464 &     0.667 \\
Accuracy on dogs                             &    0.588 &     0.729 &     0.793 &    0.462 &     0.701 \\
Accuracy on terriers                         &    0.572 &     0.704 &     0.775 &    0.421 &     0.659 \\
Accuracy for binary \{dog/not-dog\}            &    0.984 &     0.993 &     0.996 &    0.972 &     0.992 \\
Accuracy on \{terrier/not-terrier\} among dogs &    0.913 &     0.955 &     0.969 &    0.876 &     0.944 \\
\midrule
Fraction of real-terriers among dogs         &    0.224 &     0.224 &     0.224 &    0.224 &     0.224 \\
{\bf Fraction of predicted-terriers among dogs}    &    0.209 &     0.222 &     0.229 &    0.192 &     0.215 \\
\bottomrule
\end{tabular}
\caption{ImageNet classifiers are calibrated with respect to dogs:
All classifiers predict terrier for roughly $\sim22\%$ of all dogs (last row),
though they may mistake which specific dogs are terriers.}
\label{tab:imagenet}
\end{table}

\begin{figure}[t]
    \centering
    \begin{subfigure}[t]{0.4\textwidth}
        \includegraphics[width=\textwidth]{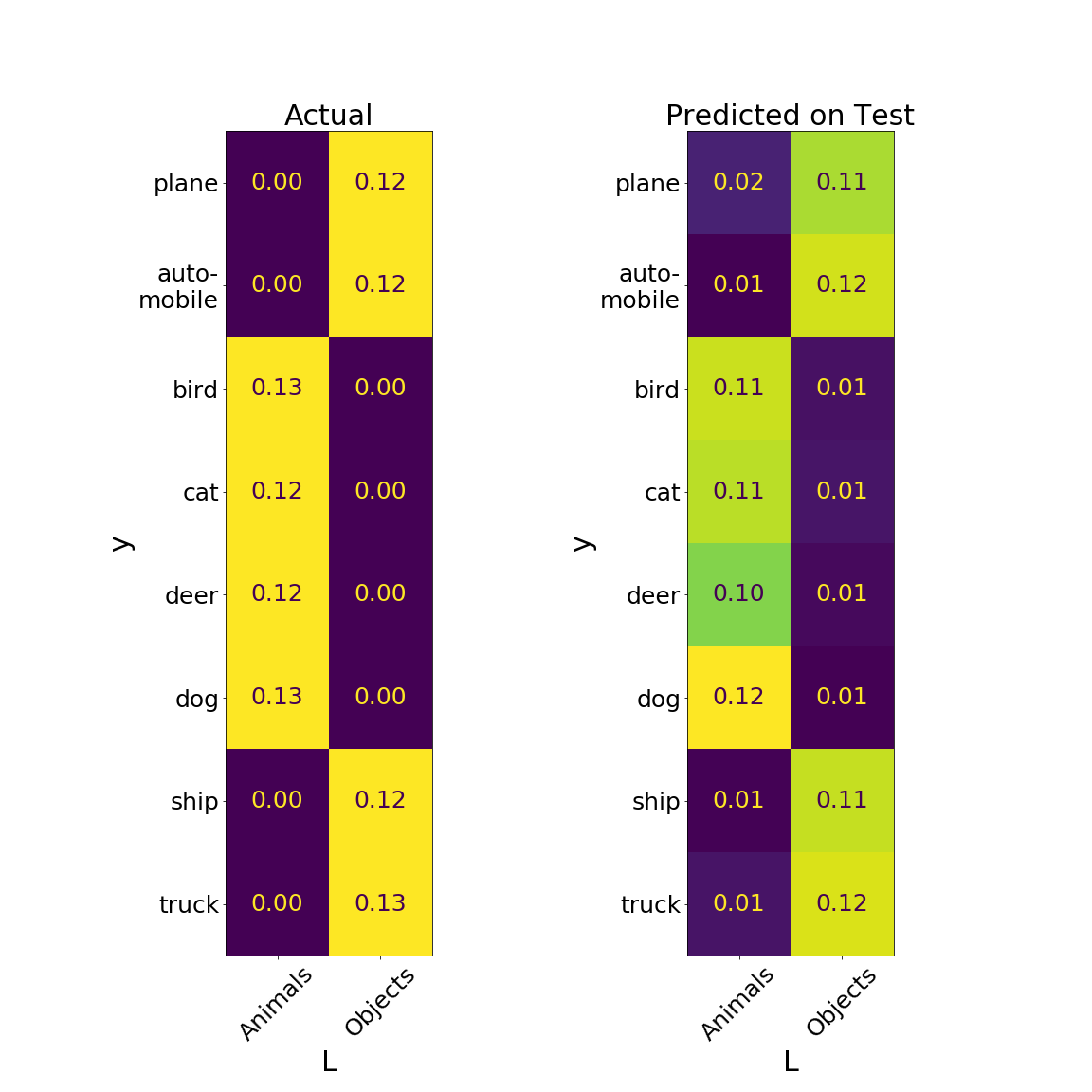}
        \caption{CIFAR10 + MLP}
        \label{fig:oa-cifar-mlp}
    \end{subfigure}
    \begin{subfigure}[t]{0.4\textwidth}
        \includegraphics[width=\textwidth]{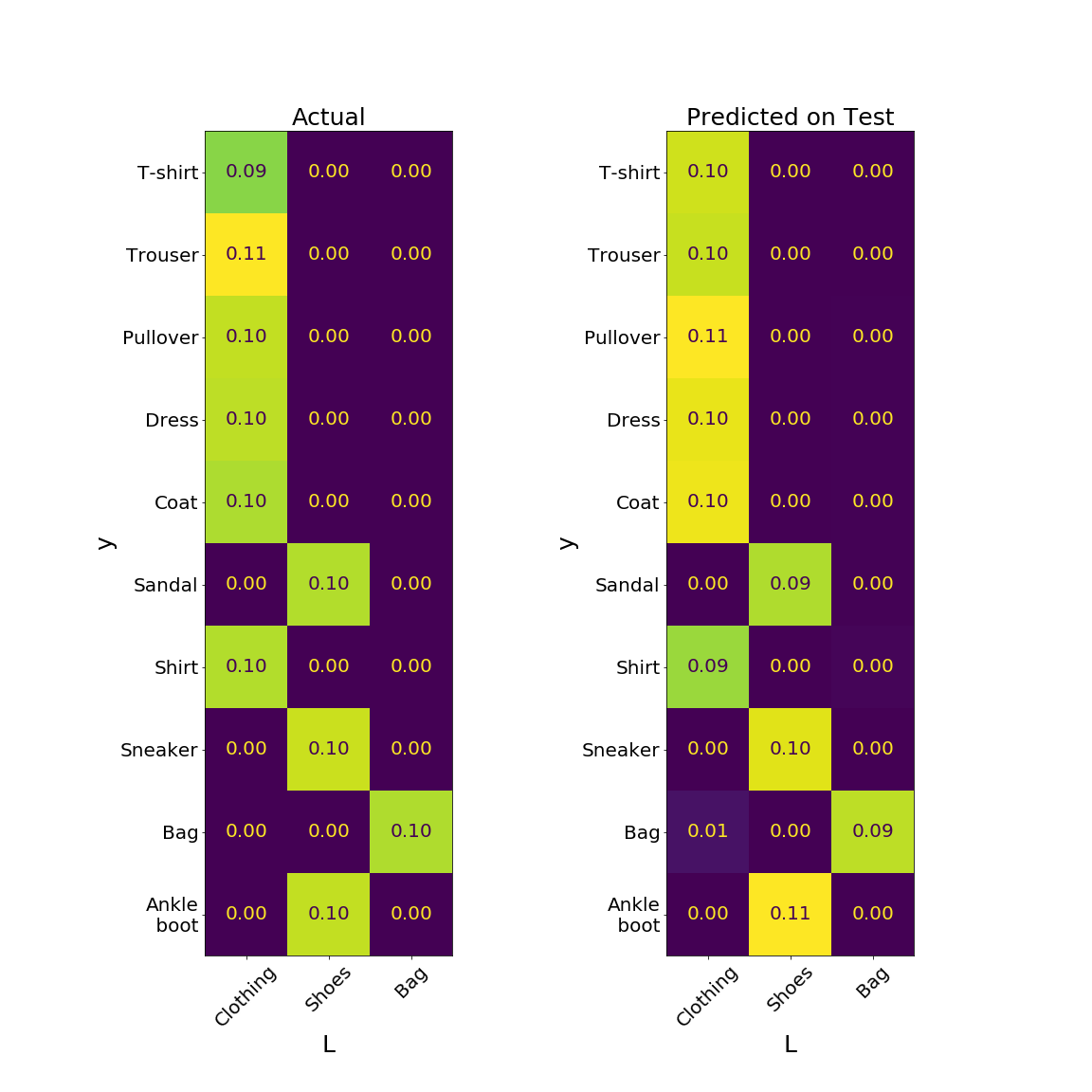}
        \caption{Fashion-MNIST + RBF}
        \label{fig:fashion}
    \end{subfigure}
    \caption{Coarse partitions as distinguishable features: We consider a setting where the original classes are not distinguishable, but the a superset of the classes are. } 
    \label{fig:coarse}
\end{figure}

\newpage
\subsection{Pointwise Density Estimation}
\label{app:ptwise}

Here we describe an informal version of the conditional-density estimation property,
and give preliminary experimental evidence to support it.
We do not yet understand this property deeply enough to state it formally, but
we believe this informal version captures the essential behavior.

\newcommand{\Lf}{L_{\textrm{fine}}}
\begin{conjecture}[Conditional Density Estimation, Informal]
\label{conj:pointwise}
For all distributions $\cD \equiv p(x, y)$,
number of samples $n$,
family of models $\cF$,
and $\eps > 0$,
let $\Lf$ be a ``finest distinguishable partition'' ---
that is, informally an $(\eps, \cF, \cD, n)$-distinguishable partition
that cannot be refined further.
Then:
$$
\text{With high probability over } x \sim \cD:
\quad
\{f(x)\}_{f \gets \Train_{\cF}(\cD^n)} \approx_\eps p(y | \Lf(x))
$$
\end{conjecture}
Conjecture~\ref{conj:pointwise} is a \emph{pointwise}
version of Conjecture~\ref{conj:approx}:
it says that for most inputs $x$,
we can sample from the conditional density $p(y | \Lf(x))$
by training a fresh classifier $f$ on iid samples from $\cD$,
and outputting $f(x)$.
We think of $\Lf(x)$ as the ``local neighborhood of
$x$'', or the finest class of $x$ as distinguishable by neural networks.
We would ideally like to sample from $p(y | x)$, but the best we can
hope for is to sample from $p(y | \Lf(x))$.

Note that this pointwise conjecture must neccesarily invoke a notion of ``finest
distinguishable partition'', which we do not formally define,
while Conjecture~\ref{conj:approx} applied to all distinguishable partitions.

\subsubsection{Experiment: Train-set Ensemble}
We now give preliminary evidence for this pointwise density estimation.
Consider the following simple distribution:
MNIST with label noise that flips class $0$ to class $1$ with probability 40\%.
To check the pointwise conjecture, we would like to train an ensemble of classifiers $f_i \gets \Train(\cD^n)$ on fresh samples from this distribution.
We do not have sufficient samples to use truly independent samples, so we approximate this as follows:
\begin{enumerate}
    \item Sample 5k random examples from the MNIST-train set.
    \item Add independent label noise (flipping $0$ to $1$ w.p. $40\%$).
    \item Train a Gaussian kernel interpolating classifier on these noisy samples (via the hyperparameters in Appendix~\ref{app:experiment}).
\end{enumerate}
We train 100 such classifiers, and let $\{f_i\}$ be the ensemble.
Then, we claim that if $x_0$ is a digit 0 in the test set,
the empirical distribution  $\{f_i(x_0)\}$ over the ensemble is
roughly 60\% label-0 and 40\% label-1.
More generally, for all test points $x$, the distribution $\{f_i(x)\}$ should be close in
total variation distance to $p(y|x)$.
In Figure~\ref{fig:ptwise-mnist} we plot the histogram of this TV distance for all $x$ in the test set
\[
H(x) := TV( ~~\{f_i(x)\}~~, ~~p(y | x)~~ )
\]
and we see that $H(x)$ is concentrated at $0$, with $\E_x[ H(x) ] \approx 0.036$.

\begin{figure}[h]
    \centering
    \includegraphics[width=0.5\textwidth]{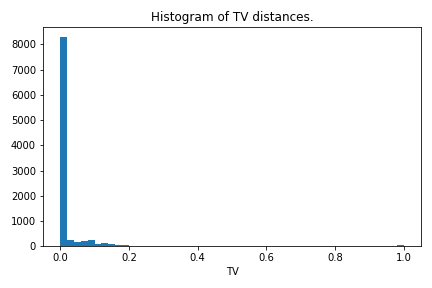}
    \caption{MNIST Ensemble.}
    \label{fig:ptwise-mnist}
\end{figure}

{\bf Discussion.}
Here we obtained a conditional density estimate of $p(y | x)$ by
training an ensemble of kernel classifiers on (approximately-) independent train sets.
In fact, we observed that for some classifiers we can \emph{re-use} the same train set,
and get the same behavior -- that is, training an ensemble with fresh random initialization alone.
In these cases, randomness in the training procedure
is somehow able to substitute for randomness in the sampling procedure.
This cannot hold for deterministic training procedures, such as kernel regression,
but we have observed it for neural-networks and decision trees.
The corresponding statement about decision trees is implicit in works on the conditional-density-estimation properties
of random forests (e.g. \citet{olson2018making}).

\newpage
\section{Agreement Property: Appendix}
\label{app:agree}

\subsection{Experimental Details}
For ResNets on CIFAR-10 and CIFAR-100, we use the following training procedure.
For $n \leq 25000$, we
sample two disjoint train sets $S_1, S_2$ of size $n$ from the 50K total train samples.
Then we train two ResNet18s $f_1, f_2$ on $S_1, S_2$ respectively.
We optimize using SGD on the cross-entropy loss, with batch size $128$,
using learning rate schedule $0.1$ for $40 \floor{\frac{50000}{n}}$ epochs,
then $0.01$ for $20 \floor{\frac{50000}{n}}$
epochs.
That is, we scale up the number of epoches for smaller train sizes, to keep the number of gradient steps constant.
We also early-stop optimization when the train loss reaches $< 0.0001$, to save computational time.
For experiments with data-augmentation, we use 
horizontal flips and \verb!RandomCrop(32, padding=4)!.
We estimate test accuracy and agreement probability on the CIFAR-10/100 test sets.

For the kernel experiments on Fashion-MNIST,
we repeat the same procedure: we sample two disjoint train sets from all the train samples,
train kernel regressors, and evaluate their agreement on the test set.
Each point on the figures correspond to one trial.

For UCI, some UCI tasks have very few examples, and so here we consider only the 92 classification
tasks from~\citet{fernandez2014we} which have at least $200$ total examples.
For each task, we randomly partition all the examples into a 40\%-40\%-20\%
split for 2 disjoint train sets, and 1 test set (20\%).
We then train two interpolating decision trees, and compare their performance on the test set.
Decision trees are trained using
a growth rule from Random Forests~\citep{breiman2001random, ho1995random}:
selecting a split based on a random $\sqrt{d}$ subset of $d$ features, splitting based on Gini impurity,
and growing trees until all leafs have a single sample.
This is as implemented by Scikit-learn~\cite{scikit-learn}
defaults with
\verb!RandomForestClassifier(n_estimators=1, bootstrap=False)!.

In Figure~\ref{fig:uci-aggr} of the body, each point corresponds to one UCI task,
and we plot the means of agreement probability and test accuracy
when averaged over 100 random partitions for each task.
Figure~\ref{fig-app:uci-single} shows the corresponding plot for a single trial.

\subsection{Additional Plots}

Figure~\ref{fig:aggr3}
shows the Laplace Kernel on Fashion-MNIST.

\begin{figure}[H]
    \includegraphics[width=0.4\textwidth]{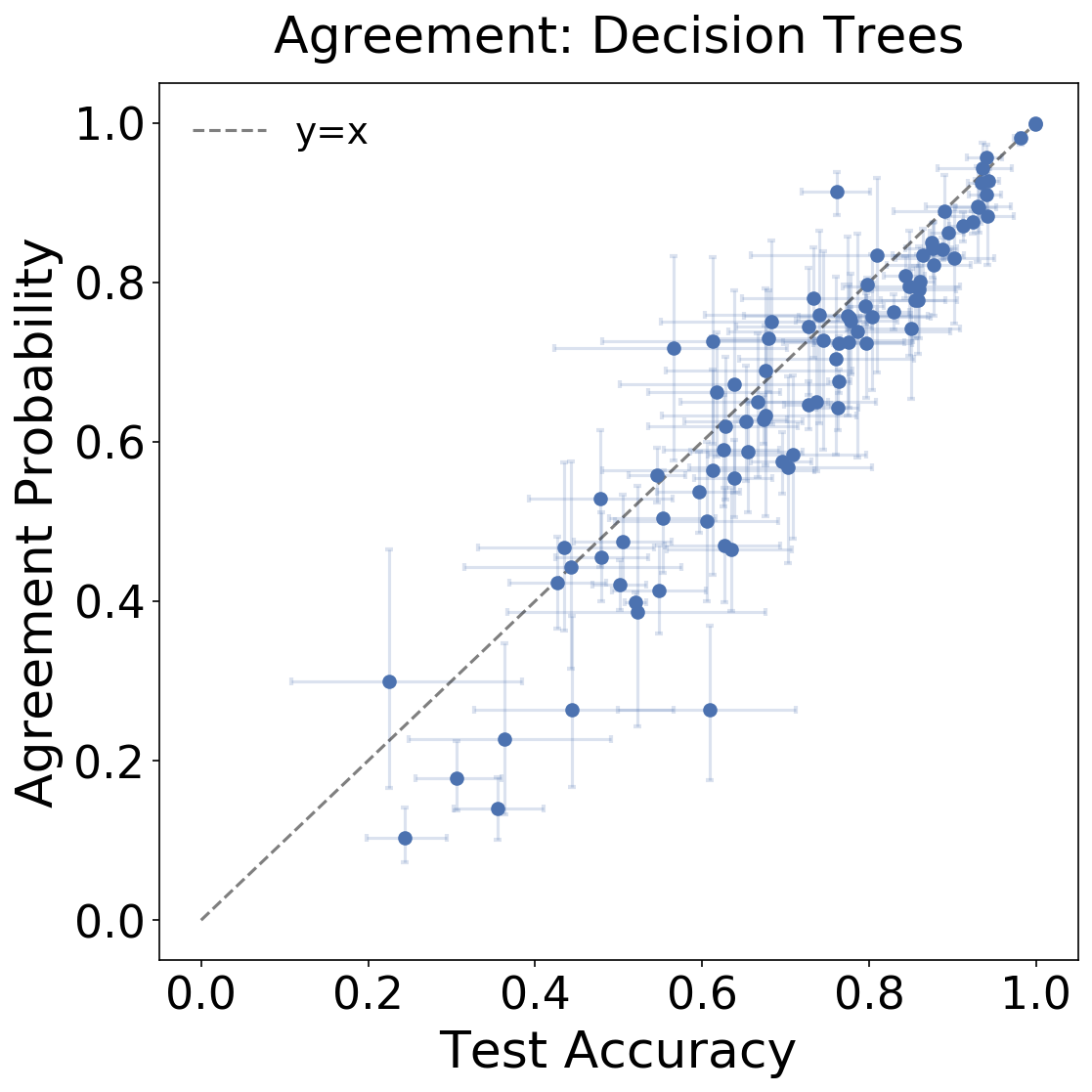}
    \centering
    \caption{Agreement Probability for a single trial of UCI classification tasks.
    Analogous to Figure~\ref{fig:uci-aggr} in the body, for a single trial.}
    \label{fig-app:uci-single}
\end{figure}

\begin{figure}[ht]
\centering
    \includegraphics[width=0.4\textwidth]{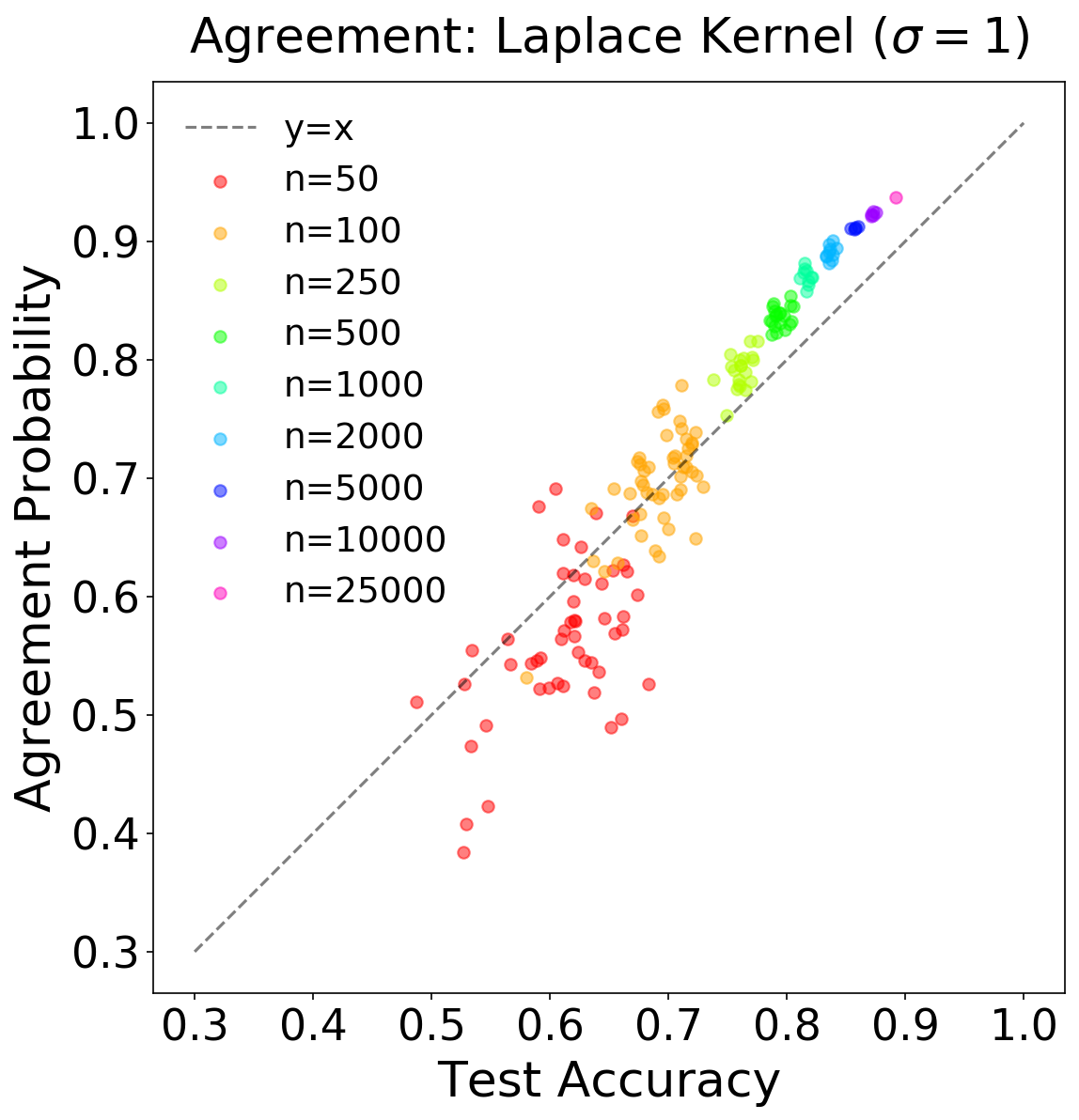}
    \caption{Laplace Kernel on Fashion-MNIST}
    \label{fig:aggr3}
\end{figure}

\subsection{Alternate Mechanisms}
\label{app:alt-mech}
We would like to understand the distribution of $f \gets \Train(\cD^n)$,
in order to evaluate the proposed mechanisms in Sections~\ref{sec:bimodal} and \ref{sec:ptwiseagree}.
Technically, sampling from this distribution
requires training a classifier on a \emph{fresh} train set.
Since we do not have infinite samples for CIFAR-10, we construct empirical
estimates by training an ensemble of classifiers on random subsets of CIFAR-10.
Then, to approximate a sample $f \gets \Train(\cD^n)$,
we simply sample from our ensemble $f \gets \{f_i\}_i$.

\subsubsection{Bimodal Samples}
Figure~\ref{fig:ptwise-corr}
shows a histogram of
$$h(x) := \Pr_{f \gets \Train(\cD^n)}[f(x) = y]$$
for test samples $x$ in CIFAR-10, where $f$ is a ResNet18 trained on 5000 samples\footnote{
We estimate this probability over the empirical ensemble $f \gets \{f_i\}_i$,
where each $f_i$ is a classifier trained on a random $5k$-subset of CIFAR-10. We train 100 classifiers in this ensemble.}.
This quantity can be interpreted as the ``easiness'' of a given test sample $(x, y)$
to a certain classifier family.

If the EASY/HARD bimodal model were true, we would expect the distribution of $h(x)$
to be concentrated on $h(x) = 1$ (easy samples)
and $h(x) = 0.1$ (hard samples).
But this is not the case in Figure~\ref{fig:ptwise-corr}.

\begin{figure}[ht]
\centering
\begin{minipage}{.5\textwidth}
  \centering
    \includegraphics[width=0.9\linewidth]{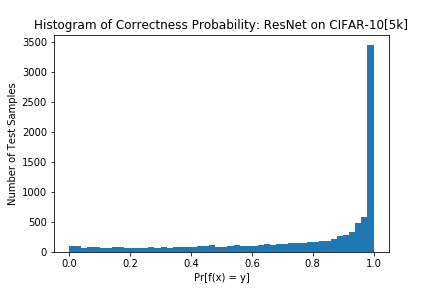}
  \captionof{figure}{Histogram of sample-hardnesses.}
  \label{fig:ptwise-corr}
\end{minipage}%
\begin{minipage}{.5\textwidth}
  \centering
   \includegraphics[width=0.9\linewidth]{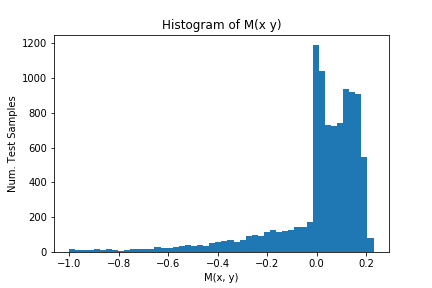}
  \captionof{figure}{Pointwise agreement histogram.}
  \label{fig:M}
\end{minipage}
\end{figure}

\subsubsection{Pointwise Agreement}
We could more generally posit
that Conjecture~\ref{claim:agree} is true
because the Agreement Property holds
\emph{pointwise} for most test samples $x$:
\begin{equation}
\text{w.h.p. for } (x, y) \sim \cD:
\quad
\Pr_{\substack{
f_1 \gets \Train(\cD^n)\\
}}[f_1(x) = y]
\approx
\Pr_{\substack{
f_1 \gets \Train(\cD^n)\\
f_2 \gets \Train(\cD^n)\\
}}[f_1(x) = f_2(x)]
\label{app-eqn:ptwise}
\end{equation}

However, we find (perhaps surprisingly) that this is not the case.
To see why this is surprising, observe that
Conjecture~\ref{claim:agree} implies that
the agreement probability is close to test accuracy,
\emph{in expectation} over the test sample and the classifiers
$f_1, f_2 \gets \Train(\cD^n)$:
\begin{align}
\Pr_{\substack{
f_1\\
(x, y) \sim \cD
}}[f_1(x) = y]
~&\approx~
\Pr_{\substack{
f_1, f_2 \\
(x, y) \sim \cD
}}[f_1(x) = f_2(x)] \tag{Conjecture~\ref{claim:agree}}\\
\iff \E_{f_1}
~
\E_{x, y \sim \cD}
[\1\{f_1(x) = y\}]
&\approx
\E_{ f_1, f_2 }
~
\E_{x, y \sim \cD}
[\1\{f_1(x) = f_2(x)\}]
\end{align}

Swapping the order of expectation,
this implies

\begin{equation}
\E_{x, y \sim \cD}
\underbrace{
\left[
\E_{f_1, f_2}[
\1\{f_1(x) = y\}
-
\1\{f_1(x) = f_2(x)\}
]
\right]
}_{M(x, y)}
\approx 0
\end{equation}
Now, we may expect that this means
$M(x, y) \approx 0$ pointwise, for most
test samples $(x, y)$.
But this is not the case.
It turns out that $M(x, y)$ takes on significantly positive and negative values,
and these effects ``cancel out'' 
in expectation over the distribution,
to yield Conjecture~\ref{claim:agree}.

For example, we compute $M(x, y)$
for the Myrtle10 kernel on CIFAR-10 with 1000 train samples.
\footnote{
We estimate the expectation in $M(x, y)$
by training an ensemble of 5000 pairs of classifiers $(f_1, f_2)$,
each pair on disjoint train samples.
}
\begin{enumerate}
    \item 
The agreement probability is within $0.8\%$
of the test error (as in Figure~\ref{fig:aggr-myrtle}), and so
$$
\E_{x,y \sim \cD}[ M(x, y) ] \approx 0.008
$$
\item However, $M(x, y)$ is not pointwise close to 0. 
E.g,
$$
\E_{x,y \sim \cD}[ \left| M(x, y) \right| ] \approx 0.133
$$
\end{enumerate}

Figure~\ref{fig:M} plots the distribution of $M(x, y)$.
We see that some samples $(x, y)$ have high agreement probability,
and some low,
and these happen to balance in expectation to yield the test accuracy.

\newpage
\section{Non-interpolating Classifiers: Appendix}
\label{app:gg}

Here we give an additional example of distributional generalization: in kernel SVM
(as opposed to kernel regression, in the main text).

\begin{figure}[H]
    \includegraphics[height=18cm]{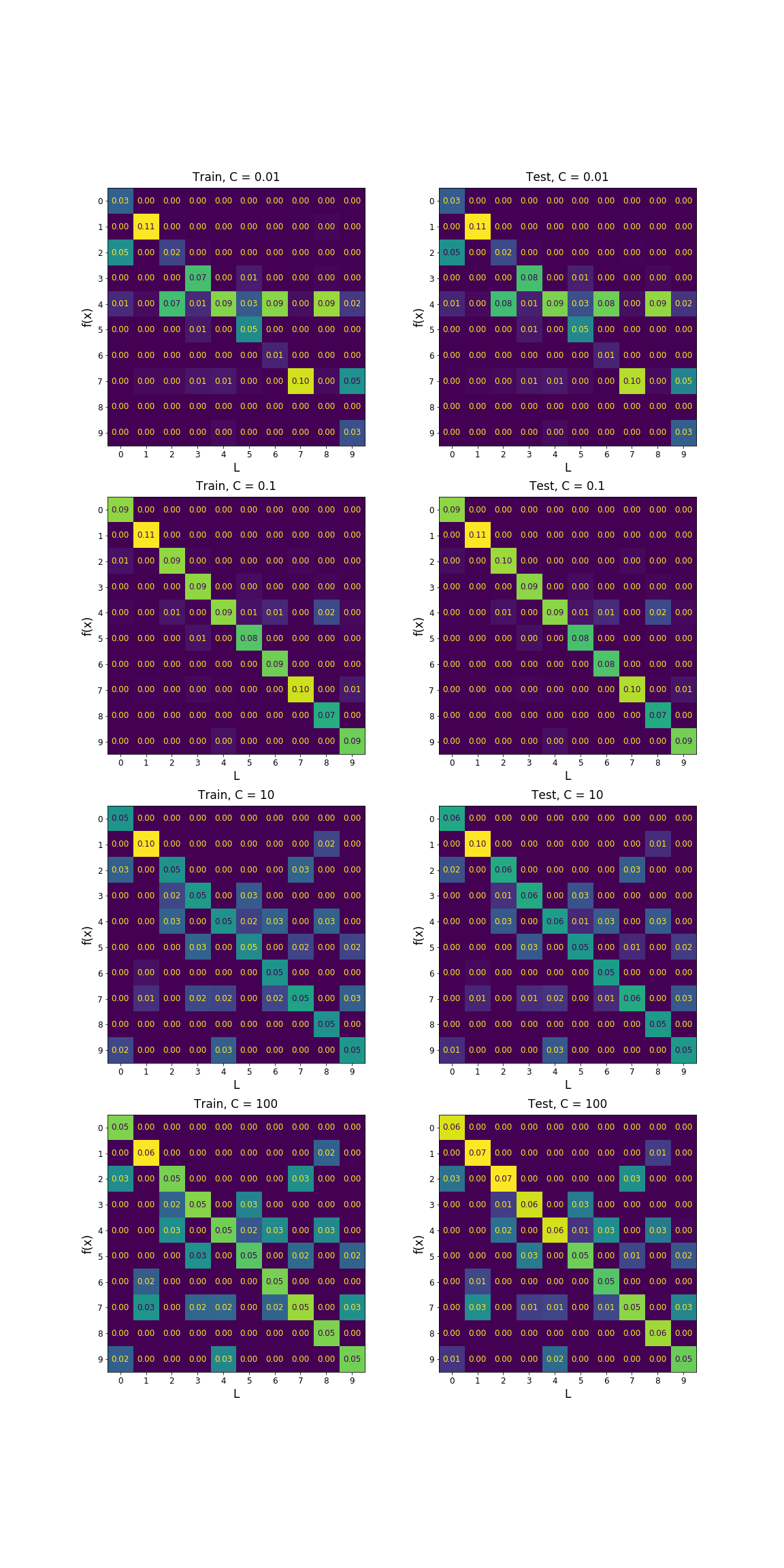}
    \centering
    \caption{{\bf Distributional Generalization.}
    Train (left) and test (right) confusion matrices for
    kernel SVM on MNIST with random sparse label noise.
    Each row corrosponds to one value of inverse-regularization parameter
    $C$. All rows are trained on the same (noisy) train set.}
    \label{fig-app:svmgg}
\end{figure}

\newpage
\section{Nearest-Neighbor Proofs}
\label{sec:proofs}

\subsection{Density Calibration Property}

\begin{proof}[Proof of Theorem~\ref{thm:Ltest}]
Recall that $L$ being an $(\eps, \NN, \cD, n)$-distinguishable partition
means that nearest-neighbors works to classify $L(x)$ from $x$:
\begin{align}
\Pr_{\substack{
\{x_i, y_i\} \sim \cD^n \\
S = \{(x_i, L(x_i)\}\\
x, y \sim \cD
}}[
\NNf_S(x) = L(x)
]
\geq 1-\eps
\end{align}
Now, we have
\begin{align}
\{(\NNf_S(x), L(x))\}_{\substack{
S \sim \cD^n\\
x, y \sim \cD
}}
&\equiv
\{( \hat{y_i}, L(x))\}_{\substack{
S \sim \cD^n\\
\hat{x_i}, \hat{y_i} \gets \NN_S(x)\\
x, y \sim \cD
}}\\
&\approx_\eps
\{( \hat{y_i}, L(\hat{x_i}))\}_{\substack{
S \sim \cD^n\\
\hat{x_i}, \hat{y_i} \gets \NN_S(x)\\
x, y \sim \cD\\
}}
\label{ln:dist}\\
&\approx_\delta
\{( \hat{y_i}, L(\hat{x_i}))\}_{\substack{
\hat{x_i}, \hat{y_i} \sim \cD
}}
\label{ln:reg}
\end{align}
Line~\eqref{ln:dist} is by distinguishability, since
$\Pr[L(x) \neq L(\hat{x_i})] \leq \eps$.
And Line~\eqref{ln:reg} is by the regularity condition.
\end{proof}

\subsection{Agreement Property}
\begin{theorem}[Agreement Property]
\label{thm:aggr}
For a given distribution $\cD$ on $(x, y)$,
and given number of train samples $n \in \N$,
suppose $\NN$ satisfies the following regularity condition:
If we sample two independent train sets $S_1, S_2$,
then the following two ``couplings'' are statistically close:
\begin{align}
\label{eqn:couple}
\{(x_i, \NN_{S_2}(x_i))\}_{\substack{
S_1 \sim \cD^n\\
S_2 \sim \cD^n\\
x_i \in_R S_1
}}
\quad\approx_\delta\quad
\{(\NN_{S_1}(x), \NN_{S_2}(x))\}_{\substack{
S_1 \sim \cD^n\\
S_2 \sim \cD^n\\
x \sim \cD
}}
\end{align}
The LHS is simply a random test point $x_i$, along with its nearest-neighbor in the train set.
The RHS produces an $(x_i, x_j)$ by sampling two independent train sets,
sampling a test point $x \sim D$,
and producing the nearest-neighbor of $x$ in $S_1$ and $S_2$ respectively.

Then:
\begin{align}
\Pr_{\substack{
S \sim \cD^n\\
(x, y) \sim \cD
}}
[\NNf_S(x) = y]
\approx_\delta
\Pr_{\substack{
S_1 \sim \cD^n\\
S_2 \sim \cD^n\\
(x, y) \sim \cD
}}
[\NNf_{S_1}(x) = \NNf_{S_2}(x)]
\end{align}
\end{theorem}

\begin{proof}
Let the LHS of Equation~\eqref{eqn:couple} be denoted as distribution
$P$ over $\cX \x \cX$.
And let $Q$ be the RHS of Equation~\eqref{eqn:couple}.
Let $\cD_x$ denote the marginal distribution on $x$ of $\cD$,
and let $p(y | x)$ denote the conditional distribution with respect to $\cD$.

The proof follows by considering the sampling of train set $S$ in the following order:
first, sample all the $x$-marginals: sample test point $x \sim \cD_x$ and train points $S_x \sim \cD_x^n$.
Then compute the nearest-neighbors $\hat{x} \gets \NN_{S_x}(x)$.
And finally, sample the \emph{values} $y$ of all the points involved, according to the densities $p(y|x)$.

\begin{align}
\Pr_{\substack{
S \sim \cD^n\\
(x, y) \sim \cD
}}
[\NNf_S(x) = y]
&=
\E_{\substack{
S \sim \cD^n \\
(x, y) \sim \cD
}}[
\1\{\NNf_S(x) = y\}
]\\
&=
\E_{\substack{
S_x \sim \cD_x^n \\
x \sim \cD_x\\
\hat{x} \gets \NN_{S_x}(x)
}}\underbrace{\left[
\E_{\substack{
y \sim p(y|x) \\
\hat{y} \sim p(y | \hat{x})
}}[
\1\{\hat{y} = y\}
]\right]}_{T(x, \hat{x})}
\\
&=
\E_{\substack{
S_x \sim \cD_x^n \\
x \sim \cD_x\\
\hat{x} \gets \NN_{S_x}(x)
}}\left[
T(x, \hat{x})
\right]\\
&=
\E_{(x_1, x_2) \sim P} T(x_1, x_2)
\tag{$P$: LHS of Equation~\eqref{eqn:couple}}\\
&\approx_\delta
\E_{(x_1, x_2) \sim Q} T(x_1, x_2)
\tag{$Q$: RHS of Equation~\eqref{eqn:couple}}\\
&=
\E_{\substack{
S_1 \sim \cD_x^n \\
S_2 \sim \cD_x^n \\
x \sim \cD_x\\
\hat{x_1} \gets \NN_{S_1}(x)\\
\hat{x_2} \gets \NN_{S_2}(x)
}}\left[
T(\hat{x}_1, \hat{x}_2)
\right]\\
&=
\E_{\substack{
S_1 \sim \cD_x^n \\
S_2 \sim \cD_x^n \\
x \sim \cD_x\\
\hat{x}_1 \gets \NN_{S_1}(x)\\
\hat{x}_2 \gets \NN_{S_2}(x)
}}
\left[
\E_{\substack{
\hat{y_1} \sim p(y |\hat{x}_1) \\
\hat{y_2} \sim p(y | \hat{x}_2)
}}[
\1\{\hat{y_1} = \hat{y_2}\}
]\right]\\
&=
\E_{\substack{
S_1 \sim \cD^n \\
S_2 \sim \cD^n \\
(x, y) \sim \cD
}}[
\1\{\NNf_{S_1}(x) = \NNf_{S_2}(x)\}
]\\
&=
\Pr_{\substack{
S_1 \sim \cD^n\\
S_2 \sim \cD^n\\
(x, y) \sim \cD
}}
[\NNf_{S_1}(x) = \NNf_{S_2}(x)]
\end{align}
as desired.
\end{proof}

\end{document}